%% file: iclr2022_conference.tex
\documentclass{article} 
\usepackage{iclr2022_conference,times}

\input{math_commands.tex}

\usepackage{hyperref}
\usepackage{url}
\usepackage{booktabs}       
\usepackage{algorithm}
\usepackage{algorithmic}
\usepackage{amsfonts}       
\usepackage{amsmath}
\usepackage{amssymb}
\usepackage{nicefrac}       
\usepackage{microtype}      
\usepackage{graphicx}       
\usepackage{wrapfig, blindtext}
\usepackage{caption}        
\usepackage{subcaption}     
\usepackage[space]{grffile} 
\usepackage{float}
\usepackage{multirow}
\usepackage{comment}
\usepackage{makecell}
\usepackage{enumitem}

\usepackage{wrapfig}
\usepackage{lipsum}

\usepackage{amsthm}
\newtheorem{thm}{Theorem}
\newtheorem{lma}{Lemma}

\newtheorem{definition}{Definition}

\title{AutoDrop: Training Deep Learning Models with Automatic Learning Rate Drop}

\author{Yunfei Teng \\
\texttt{yt1208@nyu.edu} \\
\And
Jing Wang \\
\texttt{jw5665@nyu.edu} 
\And
Anna Choromanska\\
\texttt{ac5455@nyu.edu} 
}

\begin{document}

\maketitle

\begin{abstract}
Modern deep learning (DL) architectures are trained using variants of the SGD algorithm that is run with a \textit{manually} defined learning rate schedule, i.e., the learning rate is dropped  at the pre-defined epochs, typically when the training loss is expected to saturate. In this paper we develop an algorithm that realizes the learning rate drop \textit{automatically}. The proposed method, that we refer to as AutoDrop, is motivated by the observation that the angular velocity of the model parameters, i.e., the velocity of the changes of the convergence direction, for a fixed learning rate initially increases rapidly and then progresses towards soft saturation. At saturation the optimizer slows down thus the angular velocity saturation is a good indicator for dropping the learning rate. After the drop, the angular velocity ``resets'' and follows the previously described pattern - it increases again until saturation. We show that our method improves over SOTA training approaches: it accelerates the training of DL models and leads to a better generalization. We also show that our method does not require any extra hyperparameter tuning. AutoDrop is furthermore extremely simple to implement and computationally cheap. Finally, we develop a theoretical framework for analyzing our algorithm and provide convergence guarantees. 
\end{abstract}

\section{Introduction}
\label{sec:Intro}

As data sets grow in size and complexity, it is becoming more difficult to pull useful features from them using hand-crafted feature extractors. For this reason, DL frameworks~\citep{Goodfellow-et-al-2016} are now widely popular. DL frameworks process input data using multi-layer networks and automatically find high-quality representation of complex data useful for a particular learning task. Today DL approaches are generally recognized as superior to all alternatives for image~\citep{NIPS2012_4824,He2016DeepRL}, speech~\citep{DBLP:conf/icassp/Abdel-HamidMJP12}, and video~\citep{KarpathyCVPR14} recognition, image segmentation~\citep{chen2016deeplab}, and natural language processing~\citep{DBLP:conf/emnlp/WestonCA14}. Furthermore, DL is the leading artificial intelligence technology in major tech companies such as Facebook, Google, Microsoft, and IBM, as well as in countless start-ups, where it is used for a plethora of learning problems including content filtering, photo collection management, topic classification, search/ad ranking, video search and indexing, and copyrighted material detection.

Setting the values and schedules of the hyperparameters for training DL models is computationally expensive and time consuming, e.g., a deep model with around ten billion parameters requires roughly $500$ GPUs to be trained in around two weeks~\citep{MegatronLM}. Among all hyperparameters used when training DL models, the learning rate schedule is one of the most important~\citep{jin2021autolrs}. For most SOTA DL architectures, the learning rate is dropped several times during training at epochs chosen by the user. With growing sizes of modern architectures however, performing any manual tuning of the hyperparameters will eventually become prohibitive. More efficient techniques that allow automatic and online setting of hyperparameters translate to substantial savings of resources, time, and money (today the cost of training a single state-of-the-art DL model reaches up to hundreds of thousands of dollars~\citep{Money}). 

This paper addresses a challenge of developing an automatic method for adjusting the learning rate that works in an online fashion during network training and does not introduce any extra hyper-parameters to tune. The basis for our approach is rooted in the observation that the angular velocity of the model parameters, defined below, is an excellent indicator of the dynamics of the convergence of an optimizer and can be easily used to guide the learning rate drop during network training. The resulting algorithm that we obtain is extremely simple, can be used on the top of any DL optimizer (SGD~\citep{bottou-98x}, momentum SGD~\citep{Polyak1964}, ADAM~\citep{kingma2015}, etc.), and enjoys an elegant theoretical framework. We empirically demonstrate that our method accelerates the training of DL models and leads to better generalization compared to SOTA techniques.

\begin{definition}
Define the angular velocity of model parameters as:
\vspace{-0.09in}
\begin{equation}
\omega_i = \frac{\angle(s_i,s_{i-1})}{1 \:\:\:\text{epoch}}, \:\:\:\text{where}\:\:\:s_i = x_{i+1} - x_i
\label{eq:defAV}
\end{equation}
\vspace{-0.23in}

and $x_i$ is the parameter vector in the end of the $i^{\text{th}}$ epoch. The operator $\angle(\cdot, \cdot)$ calculates the angle between two vectors and is defined as:
\vspace{-0.1in}
\begin{equation}
\angle(s_i,s_{i-1}) = \frac{180^{\circ}}{\pi} \cdot \arccos\left(\frac{s_i^T s_{i-1}}{||s_i|| ||s_{i-1}|| + \epsilon}\right),
\label{eq:DefAG}
\end{equation}
\vspace{-0.2in}

where $\epsilon$ is a small positive number preventing the division by zero \footnote{$\epsilon$ is omitted in the theoretical derivations.}.
\label{def:defAV}
\end{definition}

This paper is organized as follows: Section~\ref{sec:RW} discusses the related work, Section~\ref{sec:ME} builds an intuition for understanding our algorithm based on simple examples, Section~\ref{sec:Alg} shows our algorithm, Section~\ref{sec:Theory} captures the theoretical convergence guarantees, Section~\ref{sec:ER} presents experimental results, and Section~\ref{sec:Con} concludes the paper. All proofs and experimental details are deferred to the Supplement.

\section{Related Work}
\label{sec:RW}

In this section, we summarize different types of learning rate adaptation methods and divide them into four major categories. \textit{Scheduling-based methods} rely on a carefully designed learning rate schedules that are tailored to the non-convex nature of the deep learning optimization. More specifically, it was proposed in~\citep{smith2017cyclical} to use cyclical learning rate pattern to train DL models and apply a triangular learning rate policy in each cycle (i.e., first increase and then decrease the learning rate linearly in the cycle) to potentially allow more rapid traversal of saddle point plateaus. This idea was further extended to the super-convergence policy~\citep{smith2018super} where there is only one triangular cycle for the whole training process. This concept was also applied to other hyperparameters, e.g.:, momentum coefficient ~\citep{smith2018disciplined}. Cyclical learning rates were also used in~\citep{loshchilov2016sgdr}, where the authors combine them with restart techniques when training deep neural networks. The authors decrease the learning rate from a maximum value to a minimum value using a cosine annealing scheme and then periodically restart the process. All these methods define the learning rate policy manually, thus they constitute deterministic scheduling methods. As opposed to these techniques, ~\citep{jin2021autolrs} proposes an automatic learning rate scheduling method. The authors use Gaussian process as a surrogate model to establish the connection between the learning rate and the expected validation loss. The method updates a posterior distribution of the validation loss repeatedly and search for the best learning rate with respect to the posterior on the fly. This method requires a careful design of an acquisition function and a forecasting model in order to obtain an accurate prediction of the posterior of the validation loss.

Another group of techniques are \textit{hypergradient-based methods}~\citep{donini2020marthe, yang2019228, gunes2018online, luca2017hyper} that optimize both the model parameters and the learning rate simultaneously. The authors of these methods typically introduce a hypergradient that is defined as a gradient of the validation error with respect to the learning rate schedule. The learning rate is optimized online via gradient descent. This technique however is quite sensitive to the choice of the hyperparameters and is usually unable to reach state-of-the-art performance~\citep{jin2021autolrs}.

\textit{Hyperparameter optimization methods} aim to automatically find a good set of hyperparameters offline. They either build explicit regression models to describe the dependence of
target algorithm performance on hyperparameter settings~\citep{hutter2011}, or optimize hyperparameters by performing random search along with using greedy sequential methods based on the expected improvement criterion~\citep{bergstra2011}, or use bandit-based approach for hyperparameter selection~\citep{li2018hyperband}. These technique can be combined with Bayesian optimization~\citep{falkner2018, zela-automl18}. Recently, several parallel methods were proposed for hyperparameter tuning~\citep{jaderberg2017population, li2020population, holder2020prov, li2020system} as well. The hyperparameter optimization methods are computationally expensive in practice. 

Finally, popular \textit{adaptive learning rate optimizers} adjust the learning rate for each parameter individually based on gradient information from past iterations. AdaGrad~\citep{duchi2011} proposes to update each parameter using different learning rate which is proportional to the inverse of the past accumulated squared gradients of the parameter. Thus the parameters associated with larger accumulated squared gradients have smaller step size. This method is enabling the model to learn infrequently occurring features, as these features might be highly informative and discriminative. The major weakness of AdaGrad is that the learning rates continually decrease during the training and eventually become too small for the model to learn. Later on, RMSprop~\citep{tieleman2012lecture} and Adadelta~\citep{zeiler2012adadelta} were proposed to resolve the issue of diminishing learning rate in AdaGrad. Instead of directly summing up the past squared gradients, both methods maintain an exponential average of the squared gradients which is used to scale the learning rate of each parameter. The exponential average of the squared gradients could be considered as an approximation to the second moment of the gradients. One step further, ADAM~\citep{kingma2015} estimates both first and second moments of the gradients and use them together to update the parameters.

\section{Motivating Example}
\label{sec:ME}

In this section we analyze the properties of the angular velocity for a noisy quadratic model. While simple, this model is used as a proxy for
analyzing neural network optimization~\citep{schaul2013pesky,pmlr-v37-martens15,zhang2019lookahead}.

\begin{definition}[Noisy Quadratic Model]
We use the same model as in~\citep{zhang2019lookahead}. The model is represented by the following loss function
\vspace{-0.05in}
\begin{align}\label{eq:loss}
    L(x)=\frac{1}{2}(x-c)^TA(x-c),
\end{align}
\vspace{-0.2in}

where $c\sim N(x^*,\Sigma)$ and both $A$ and $\Sigma$ are diagonal. Without loss of generality, we assume $x^*=0$. 
\label{def:NQM}
\end{definition}

The update formula for the gradient descent at the step $t+1$  is given as
\vspace{-0.05in}
\begin{align}
    x_{t+1}=x_t-\alpha\nabla L(x_t)=x_t-\alpha A(x_t-c_t),\quad c_t\sim N(0,\Sigma),
\end{align}
\vspace{-0.25in}

where $\alpha$ is the learning rate.

We optimize noisy quadratic model with $x\in\mathbb{R}^{200}$ and $A=diag(\frac{1}{10},\frac{2}{10},...,\frac{200}{10})$ using Gradient Descent (GD), where in each experiment $\alpha = [0.06, 0.03, 0.01, 0,001]$.

\begin{figure}[htp!]
\vspace{-0.15in}
\centering
\includegraphics[trim={0 0 0 1cm},clip,width=.45\textwidth]{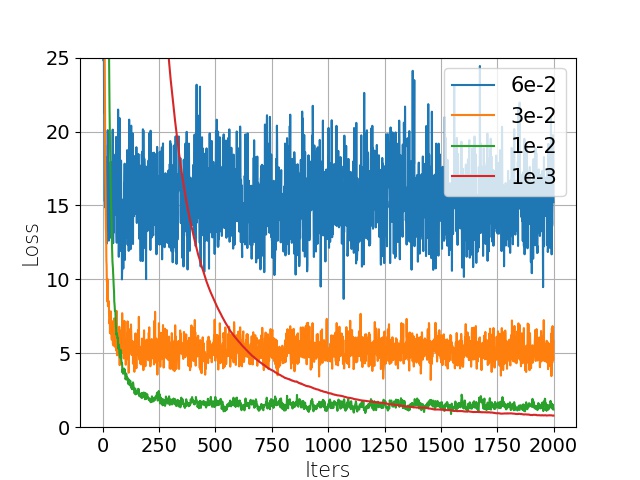}
\includegraphics[trim={0 0 0 1cm},clip,width=.45\textwidth]{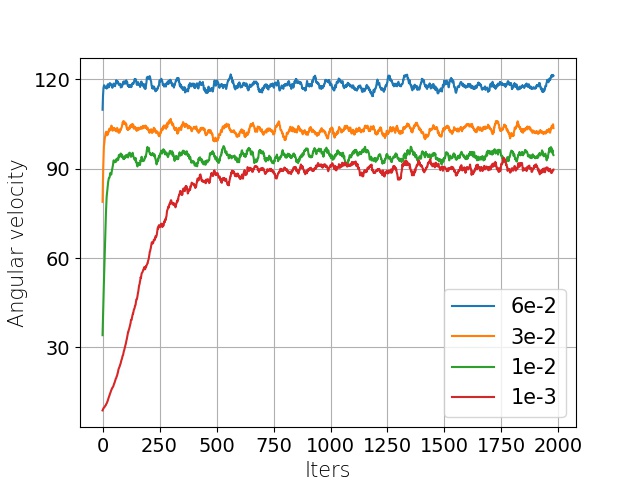}
\vspace{-0.14in}
\caption{The behavior of the loss and angular velocity for noisy quadratic model. An optimizer is run with different settings of the learning rate $\alpha=[0.06, 0.03, 0.01, 0,001]$. Angular velocity is averaged over $20$ iterations.}
\label{fig:NQM}
\vspace{-0.1in}
\end{figure}

The experiments captured in Figure~\ref{fig:NQM} reveal the following properties:
\begin{itemize}[noitemsep,topsep=0pt,parsep=0pt,partopsep=0pt, leftmargin=0.3in]
    \item [(P1)] \textbf{Angular velocity saturation:} the angular velocity curves\footnote{For the noisy quadratic model, the angular velocity (given in Definition~\ref{def:defAV}) is computed with respect to one iteration, rather than an epoch, as for this model there is no notion of the epoch.} have the tendency to saturate as the training proceeds, and furthermore when the angular velocity enters the saturation phase, the optimizer slows down its convergence,
    \item [(P2)] \textbf{Angular velocity saturation levels:} i) if the learning rate is large enough such that the algorithm cannot converge to the optimum, the angular velocity saturates at a level larger than $90$ degrees and smaller than $120$ degrees; ii) as the learning rate decreases, and the algorithm systematically converges closer to the optimum, the angular velocity saturates at progressively lower levels; iii) smaller learning rate leads to a slower saturation of the angular velocity; iv) when the learning rate is low enough such that the algorithm can converge to the optimum, the angular velocity saturates at $90$ degrees.  
\end{itemize}

These empirical properties can be theoretically justified as shown in the next theorem. 

\begin{thm}\label{thm_1}
Let the $i$-th diagonal terms of matrices $A$ and $\Sigma$ in the noisy quadratic model be given as $a_i$ and $\sigma_i$, respectively. Then, the expected inner product $<s_{t},s_{t+1}>$ converges to
\vspace{-0.05in}
\begin{align}
    I^*=\lim_{t\to\infty}\mathbb{E}[<s_{t},s_{t+1}>]=-\alpha^3\sum_{i=1}^n\frac{a_i^3\sigma_i^2}{2-\alpha a_i}.
\end{align}
\vspace{-0.2in}

Moreover, the cosine value of an angle between two consecutive steps $\cos\angle (s_t,s_{t+1})$ satisfies
\vspace{-0.05in}
\begin{align}
      C^*\!\!=\!\!\lim_{t\to\infty}\mathbb{E}[cos(\angle(s_t,s_{t+1}))]\!\approx\!-\frac{\alpha}{2}\frac{\sum_{i=1}^n\frac{a_i^3\sigma_i^2}{2-\alpha a_i}}{\sum_{i=1}^n\frac{a_i^2\sigma_i^2}{2-\alpha a_i}}\!\geq\!-\frac{\alpha\max_i a_i}{2}
  \end{align}
  \vspace{-0.2in}

 $C^*\in [-\frac{1}{2}, 0]$ and thus $\angle (s_t,s_{t+1})$ is between $90$ to $120$ degrees.
\end{thm}

Theorem~\ref{thm_1} implies that as training proceeds, the angular velocity eventually saturates as stated in property P1.
Theorem~\ref{thm_1} furthermore shows that decreasing the learning rate causes the angle between $s_{t}$ and $s_{t+1}$ to converge to a smaller value. Also, from Theorem \ref{thm_1}, $I^*=\lim_{t\to\infty}\mathbb{E}[<s_{t},s_{t+1}>]=-\sum_{i=1}^n(\alpha a_i)^3\sigma_i^2\left[\frac{1}{2-\alpha a_i}\right]$. When $\alpha a_i(i=1,..,n)$ is small enough, $I^*$ can be treated as $0$ which implies that $s_t$ is orthogonal to $s_{t+1}$. In other words, the angle between $s_{t}$ and $s_{t+1}$ converges to $90$ degrees for small enough learning rate. Otherwise, for larger learning rates, this angle saturates above $90$ degrees. Furthermore, the limit of cosine angle $C^*$ is approximately larger than $-\frac{1}{2}$, thus the saturation level of angular velocity should be below $120$ degrees. This together supports property P2 (in particular this supports points i,ii, and iv; point iii remains an empirical observation).

We next empirically verified whether these observations carry over to non-convex DL setting on a simple experiment reported in Figure~\ref{fig:DeepME}. Clearly, property P1 holds, whereas property P2 is satisfied partially. In particular conclusion iii is broken as the angular velocity may not reach $90$ degrees. Also, in a DL setting one can observe that for lower learning rates the angular velocity curves become more noisy at saturation, which was not the case for a noisy quadratic model. 

\begin{figure}[htp!]
\vspace{-0.1in}
\centering
\includegraphics[width=.325\textwidth]{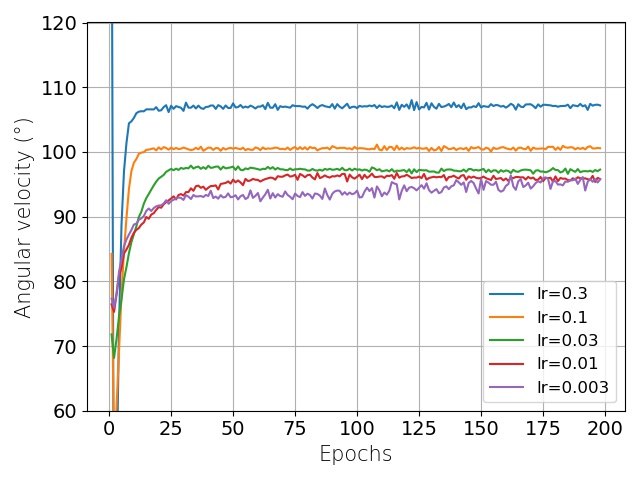}
\includegraphics[width=.325\textwidth]{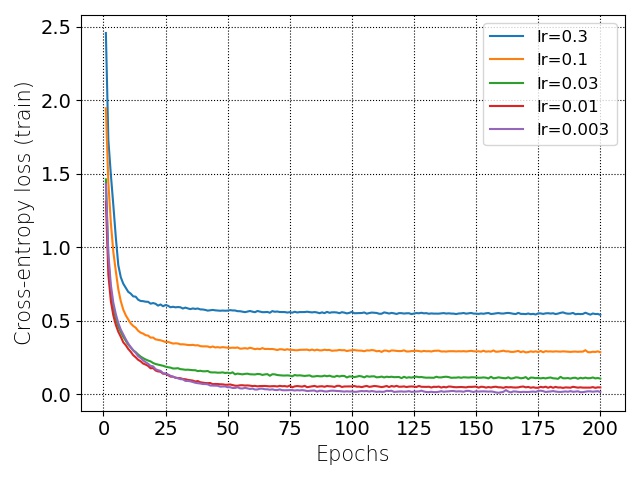}
\includegraphics[width=.325\textwidth]{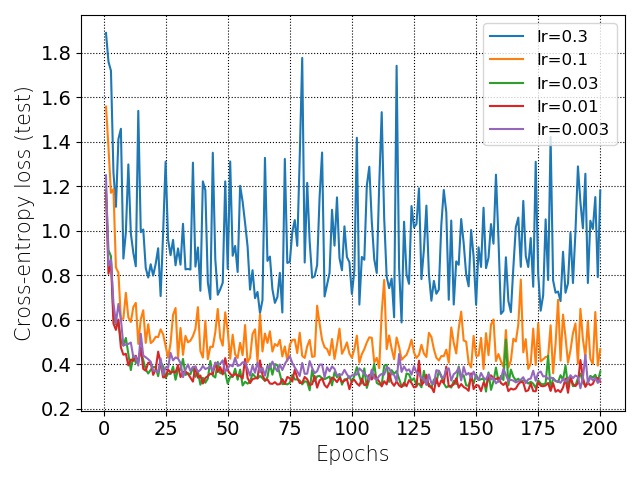}
\vspace{-0.14in}
\caption{The behavior of the loss and angular velocity for an exemplary DL problem (training ResNet-18 on CIFAR-10). An optimizer is run with different settings of the learning rate $\alpha=[0.3, 0.1, 0.03, 0.01, 0.003]$. Angular velocity is calculated over a single epoch.}
\label{fig:DeepME}
\vspace{-0.1in}
\end{figure}

Property P1 is a key observation underlying our algorithm. An important conclusion from this observation is that the saturation of the angular velocity can potentially guide the drop of the learning rate of the optimization algorithm. In other words, given the lower-bound on the learning rate, each time the angular velocity saturates, the learning algorithm should decrease the learning rate. Tracking the saturation of the angular velocity is more plausible than tracking the saturation of the loss function since, as can be clearly seen in Figure~\ref{fig:NQM}, angular velocity curves follow much harder saturation pattern. Also, the loss function does not necessary need to have a bounded range, as opposed to the angular velocity. We found that property P1 is sufficient to design an optimization algorithm for training DL models. The algorithm is described in Section~\ref{sec:Alg}. Property P2 is crucial for the theoretical analysis provided in Section~\ref{sec:Theory}.

Following the above intuition, we implement a simple algorithm for optimizing the noisy quadratic model. The algorithm drops the learning rate by a factor of $2$ when the angular velocity saturates (i.e.:, the change of the angular velocity averaged across $20$ iterations is smaller than $0.01$ degree between $2$ consecutive iterations). The initial learning rate was set to $0.06$ and the minimal one was set to $0.001$. Figure~\ref{fig:NQM2} captures the results. It shows that the algorithm that is using the angular velocity to guide the drop of the learning rate indeed converges to the optimum.

\begin{wrapfigure}{r}{0.7\textwidth} 
\vspace{-0.2in}
\centering
\includegraphics[trim={0.25cm 0 0.3cm 1cm},clip,width=0.35\textwidth]{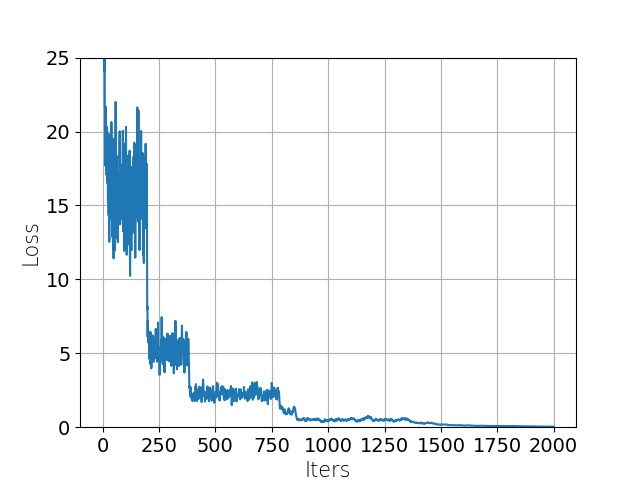}
\hspace{-0.05in}\includegraphics[trim={0.2cm 0 0.3cm 1cm},clip,width=0.35\textwidth]{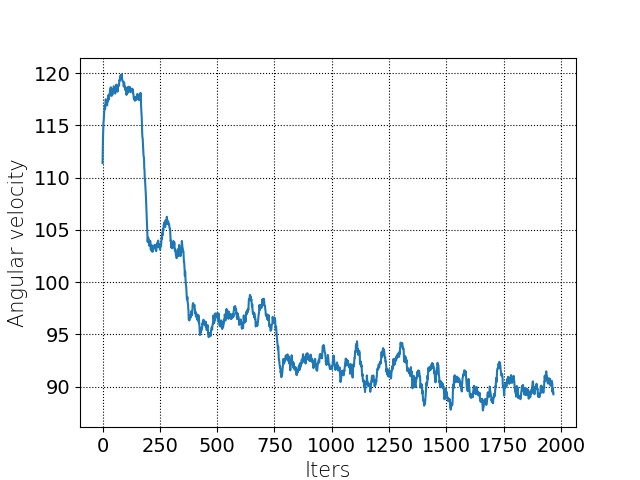}
\vspace{-0.25in}
\caption{The behavior of the loss and angular velocity for the noisy quadratic model. An optimizer is using an automatic drop of the learning rate guided by the saturation of the angular velocity. Angular velocity is averaged over $20$ iterations.}
\label{fig:NQM2}
\vspace{-0.15in}
\end{wrapfigure}

The aforementioned simple algorithm led us to derive the method for optimizing DL models using automatic learning rate drop that we refer to as AutoDrop. The obtained method is a straightforward extension of the above algorithm and is described in the next section. The extension accommodates the fundamental difference that we observed between noisy quadratic model and the DL model: the fact that in the case of DL models, lower learning rates lead to a larger noise of the angular velocity at saturation.

\begin{algorithm} 
\caption{AutoDrop} 
\label{alg:ALRD}
\begin{algorithmic}
\REQUIRE 
\STATE $\alpha_0$ and $\underline{\alpha}$: initial learning rate of the optimizer and its lower bound\\
$\theta_0$ and $\overline{\theta}$: initial threshold for the change in the angular velocity and its upper bound\\
$\rho$: learning rate drop factor\\
$x_0$ : initial model parameter vector\\
$n_d$: learning rate drop delay in number of epochs\\
\STATE
\STATE $x \leftarrow x_0$, $\alpha \leftarrow \alpha_0, \theta \leftarrow \theta_0$, $s_0 \leftarrow 0$, $t \leftarrow 0$
\STATE $drop\alpha \leftarrow false$
\WHILE{not converged}
\STATE $t \leftarrow t + 1$, $y \leftarrow x$
\STATE Train the model for one epoch with learning rate $\alpha$ and update $x$ accordingly 
\STATE $s_t \leftarrow x - y$; $\omega_t \leftarrow \angle(s_t,s_{t-1})$
\STATE
\STATE
\textsf{//Check the condition for dropping $\alpha$}
\IF{$t > 2$ and $|\omega_t - \omega_{t-1}| < \theta$}
\STATE $z \leftarrow 0$, $k \leftarrow 0$, $drop\alpha \leftarrow true$
\ENDIF
\STATE
\IF{$drop\alpha$}
\STATE $k \leftarrow k + 1$
\STATE $z \leftarrow z + k\cdot x$ \:\:\:\textsf{//Accumulate the scaled values of parameters}
\IF{$k \geq n_d$}
\STATE $x \leftarrow \frac{2z}{(n_d+1)n_d}$ \:\:\:\textsf{//Compute the exponential average of model parameters}\\
\STATE $\alpha \leftarrow \max\{\underline{\alpha}, \rho \times \alpha\}$, $\theta \leftarrow \min\{\overline{\theta}, \frac{1}{\rho} \times \theta\}$ \:\:\:\textsf{//Drop $\alpha$ and adjust $\theta$}\\
\STATE $drop\alpha \leftarrow false$
\ENDIF
\ENDIF
\ENDWHILE
\STATE
\STATE
\textsf{//Recommended setting of hyperparameters:}\\ $\underline{\alpha} = 0.0001$, $\theta_0 = 0.01^{\circ}$, $\overline{\theta} = 1^{\circ}$, and $n_d = 20$\\
$\rho, \alpha_0$ - the same as in SOTA
\end{algorithmic}
\end{algorithm}

\section{Algorithm}
\label{sec:Alg}

The algorithm for training deep learning models with automatic learning rate drop is captured in Algorithm~\ref{alg:ALRD}. The algorithm admits on its input the initial learning rate $\alpha_0$, the value of the smallest permissible learning rate $\underline{\alpha}$, initial threshold for the change in the angular velocity $\theta_0$ that will determine the first drop of the learning rate, the value of the largest permissible threshold for the change in the angular velocity $\overline{\theta}$, learning rate drop factor $\rho$ ($\rho \in (0,1)$; each time the learning rate is dropped, it is multiplied by $\rho$), initial model parameter vector $x_0$, and the learning rate drop delay $n_d$ (this hyper-parameter will be explained in the next paragraph).

The algorithm triggers the procedure for dropping the learning rate each time the angular velocity changes by less than the threshold $\theta$ between two consecutive epochs ($\theta$ is initialized with $\theta_0$). Before the learning rate is dropped (i.e., multiplied by $\rho$), the optimizer continues operating with the current learning rate for another $n_d$ epochs during which it calculates the exponential average of model parameters (parameter averaging is commonly done by practitioners and was proposed by~\citep{Polyak1992AccelerationOS}). This is done to stabilize the learning process. Finally, after each learning rate drop, the threshold for the change in the angular velocity is increased (i.e., divided by $\rho$). This is necessary as the angular velocity becomes more noisy for the lower learning rates.

AutoDrop algorithm can be thought of as a meta-scheme that can be put on the top of any optimization method for training deep learning models. Thus one can use any optimizer to update model parameters. In practice we recommend using the following setting of the hyperparameters for our algorithm: $\underline{\alpha} = 0.0001$, $\theta_0 = 0.01^{\circ}$, $\overline{\theta} = 1^{\circ}$, $n_d = 20$, and $\rho$ set in the same way as in SOTA. As will be shown in the experimental section this set of parameters guarantees good performance for a wide range of model architectures and data sets. 

\section{Theory}
\label{sec:Theory}

This section theoretically shows that decreasing the learning rate when the angular velocity saturates guarantees the sub-linear convergence rate of SGD and momentum SGD methods.

\subsection{Unified convergence analysis for SGD and momentum SGD with discrete learning rate drop}
Firstly, we present a unified theoretical framework that covers the update rule of both SGD and momentum SGD. We refer to these update rules jointly as Unified Momentum (UM) method. This framework was proposed in~\citep{yang2016unified}.

\vspace{-0.2in}
\begin{equation}
\text{UM}:\quad\left\{
\begin{aligned}\label{eq:sgdm}
      y_{t+1}&=x_t-\alpha_t\mathcal{G}(x_t;\xi_t)\\
      y_{t+1}^s&=x_t-s\alpha_t\mathcal{G}(x_t;\xi_t)\\
      x_{t+1}&=y_{t+1}+\beta(y_{t+1}^s-y_t^s)
\end{aligned}
\right.    
\end{equation}
\vspace{-0.18in}

where $t$ is the iteration index, $\beta$ is the momentum parameter, $\alpha_t$ is the learning rate at time $t$, $x_t$ is the parameter vector at time $t$, and $\mathcal{G}(x_t;\xi_t)$ is the gradient of the loss function at time $t$ computed for a data mini-batch $\xi_t$. $s$ is the factor that controls the type of optimization method in the following way:

\vspace{-0.05in}
\begin{itemize}[noitemsep,topsep=0pt,parsep=0pt,partopsep=0pt, leftmargin=0.3in]
    \item $s=0$ Heavy-Ball (HB) method:
    \vspace{-0.05in}
    $$\text{HB:}\quad x_{t+1}=x_t-\alpha_t\mathcal{G}(x_t;\xi_t)+\beta(x_t-x_{t-1})$$
    \vspace{-0.2in}
        
    \item $s=1$ Nestrov (NAG) method:
    \vspace{-0.1in}
    $$\text{NAG:}\quad \quad\left\{
    \begin{aligned}  
          y_{t+1}&=x_t-\alpha_t\mathcal{G}(x_t;\xi_t)\\
          x_{t+1}&=y_{t+1}+\beta(y_{t+1}-y_t)
    \end{aligned}
    \right.$$
    \vspace{-0.1in}
     
    \item $s=1/(1-\beta)$ Gradient Descent (GD) method:
    \vspace{-0.05in}
    $$\text{GD:}\quad x_{t+1}=x_t-\alpha_t/(1-\beta)\mathcal{G}(x_t;\xi_t).$$
    \vspace{-0.2in}
\end{itemize}
\vspace{-0.05in}

The state-of-the-art convergence analysis for common machine learning optimizers only supports constant learning rate~\citep{le2012stochastic, yang2016unified, schmidt2017minimizing, ramezani2018stability, zhang2019fast} or continuous learning rate drop schemes~\citep{wu2018wngrad, wu2019global, gower2019sgd}. However, the learning rate is dropped in a discrete fashion in many practical cases, especially in DL. Theorem~\ref{thm:sgdm_conv} provides a theoretical convergence guarantee for optimization algorithms that use discrete learning rate drop. The theorem requires some mild (easy to satisfy in practice and thus realistic) constraints on the drop gap ($k_i$), i.e.:, the frequency of dropping the learning rate. Theorem~\ref{thm:sgdm_conv} accommodates learning settings relying on discrete learning rate drops and thus is well-aligned with approaches used by practitioners. Moreover, in the next section we extend this theorem to our AutoDrop approach.

\begin{thm}\label{thm:sgdm_conv}
Suppose $f(x)$ is a convex function, $\mathbb{E}\left[\norm{ \mathcal{G}(x;\xi)-\mathbb{E}[\mathcal{G}(x;\xi)]}\right]\leq\delta^2$ and $\norm{\partial f(x)}\leq G$ for any $x$ and some non-negative $G$. Given a sequence of decreasing learning rates $\{\hat{\alpha}_i\}_{i=-1}^{n-1}\subset (0,1)$ and a sequence of integers $\{k_i\}_{i=0}^{n-1}\subset \mathbb{N}$ ($n\gg 1$), there exits constants $\kappa_1,\kappa_2$ such that
\vspace{-0.07in}
\begin{align}\label{ieq:ki_con2}
    \hat{\alpha}_i\leq (i+2)^{-1}, \quad k_i\hat{\alpha}_i\geq \kappa_1,\quad k_i\hat{\alpha}_i\hat{\alpha}_{i-1}\leq \kappa_2(i+1)^{-1},\quad \forall i=0,1,...,n-1.
\end{align}
\vspace{-0.23in}

Define a partition $\Pi:0=t_0<t_1<...<t_n=T (T=\sum_{i=0}^{n-1}k_i)$ based on the integer sequence $\{k_i\}_{i=0}^{n-1}$ such that the gap between $t_i$ and $t_{i+1}$ is $k_i$ ($k_i = t_{i+1}-t_i$). Run UM update defined in Equation~\ref{eq:sgdm} for $T$ iterations by setting the learning rate $\alpha_t$ based on a sequence $\{\hat{\alpha}_i\}_{i=-1}^{n-1}$ as
\vspace{-0.1in}
\begin{align}
    \alpha_t=\hat{\alpha}_i,\quad\text{where }t_i\leq t< t_{i+1}.
\end{align}
\vspace{-0.23in}

Then the following holds:
\vspace{-0.1in}
\begin{align*}
    \min_{t=0,...,T-1}\{\mathbb{E}[f(x_t)-f(x^*)]\}\leq&\frac{\beta(f(x_0)-f(x^*))[\log(n+1)-\log 2]}{\kappa_1(1-\beta)n}+\frac{(1-\beta)\norm{x_0-x^*}^2}{2\kappa_1n}\notag\\
    &+\frac{(2s\beta+1)(G^2+\delta^2)\kappa_2\log n}{2(1-\beta)\kappa_1n}.
\end{align*}
\vspace{-0.23in}
\end{thm}

\subsection{Convergence Analysis of AutoDrop}\label{subsec:thm_auto_drop}
For a fixed learning rate $\alpha$, we introduce a simplified mathematical model of the behavior of the angular velocity as a function of iterations. The model is defined below (and depicted in Figure~\ref{fig:flat}): 
\vspace{-0.1in}
\begin{equation}
    v_{\alpha}(t)=\frac{\pi}{2}(1+\epsilon\alpha)\left(1-\frac{1}{\gamma\alpha (t+1/\gamma \alpha)}\right),
    \label{eq:AVModel}
\end{equation}
\vspace{-0.23in}

where $t$ is the number of iterations, $\epsilon$ and $\gamma$ are two constants that control the asymptote and curvature of the velocity. 

\makeatletter\def\@captype{figure}\makeatother
\vspace{-0.25in}
\begin{minipage}{.38\textwidth}
\centering
\begin{figure}[H]
    \centering
    \includegraphics[width=\textwidth]{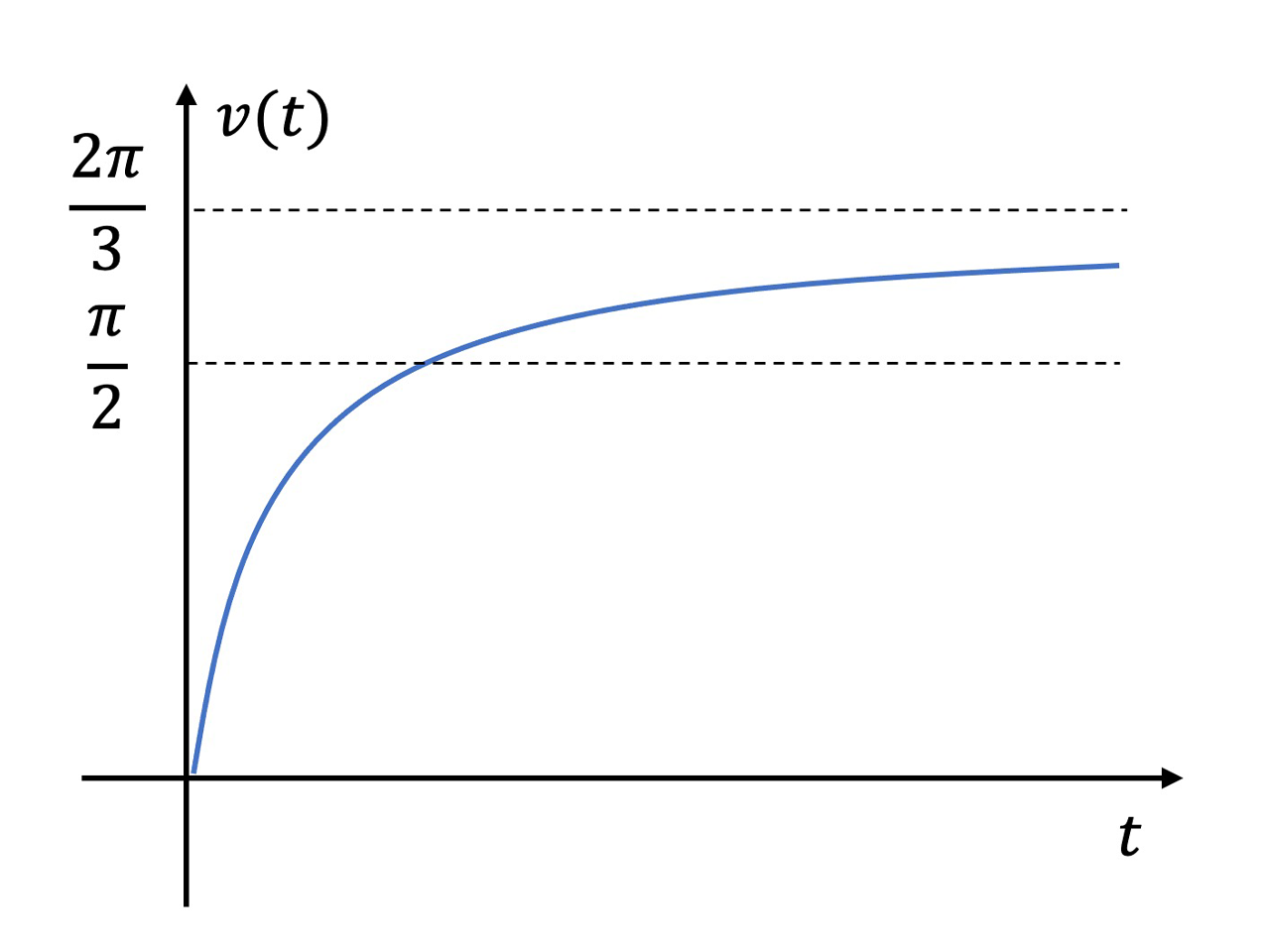}
    \vspace{-0.15in}
    \caption{Angular velocity model for a fixed learning rate $\alpha$.}
    \label{fig:flat}
\end{figure}
\vfill
\end{minipage}
\makeatletter\def\@captype{table}\makeatother
\begin{minipage}{.61\textwidth}
\centering
\begin{algorithm}[H]
    \centering
    \caption{AutoDrop (approximate)}\label{alg:LRdrop}
    \begin{algorithmic}
    \STATE \textbf{Inputs:} $x_0$: initial weight \\ 
    \STATE \textbf{Hyperparameters:} $\{\hat{\alpha}_i\}$: set of learning rates, $v_{\alpha}(t)$: ang. vel. model, $\tau_0$: init. threshold for the derivative of ang. vel. \\ 
    \STATE Initialize $i=0$, $t_0=0$, $t=0$\\
    \WHILE{$i<n$}
        \STATE Update $x_t$ via (\ref{eq:sgdm}) with learning rate $\alpha_t\!=\!\hat{\alpha}_i$.
        \IF{$v'_{\hat{\alpha}_i}(t-t_i)\leq\tau_i=\min\{\tau_0, \gamma \hat{\alpha}_i/2\}$}
          \STATE $i=i+1; t_i=t$
        \ENDIF
        \STATE $t=t+1, T=t$
    \ENDWHILE
    \RETURN $\{x_t\}_{t=0}^{T-1}$ (T: $\#$ iterations)
    \end{algorithmic}
\end{algorithm}
\end{minipage}

\begin{wrapfigure}{r}{0.61\textwidth} 
\vspace{-0.25in}
\begin{center}
\includegraphics[trim={0.3cm 0 0.3cm 2cm},clip,width=.3\textwidth]{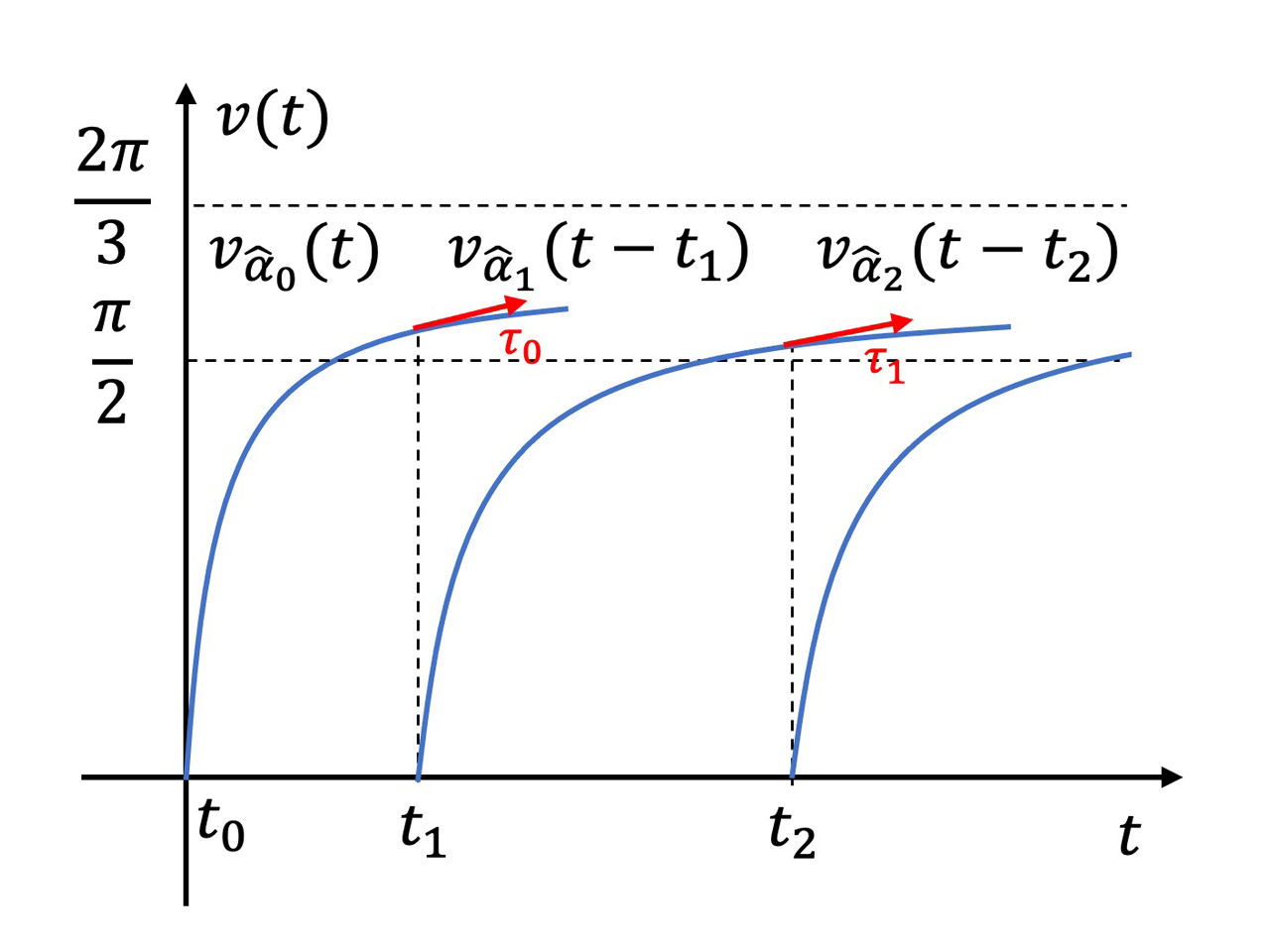}
\includegraphics[trim={0.45cm 0 0.3cm 2cm},clip,width=.3\textwidth]{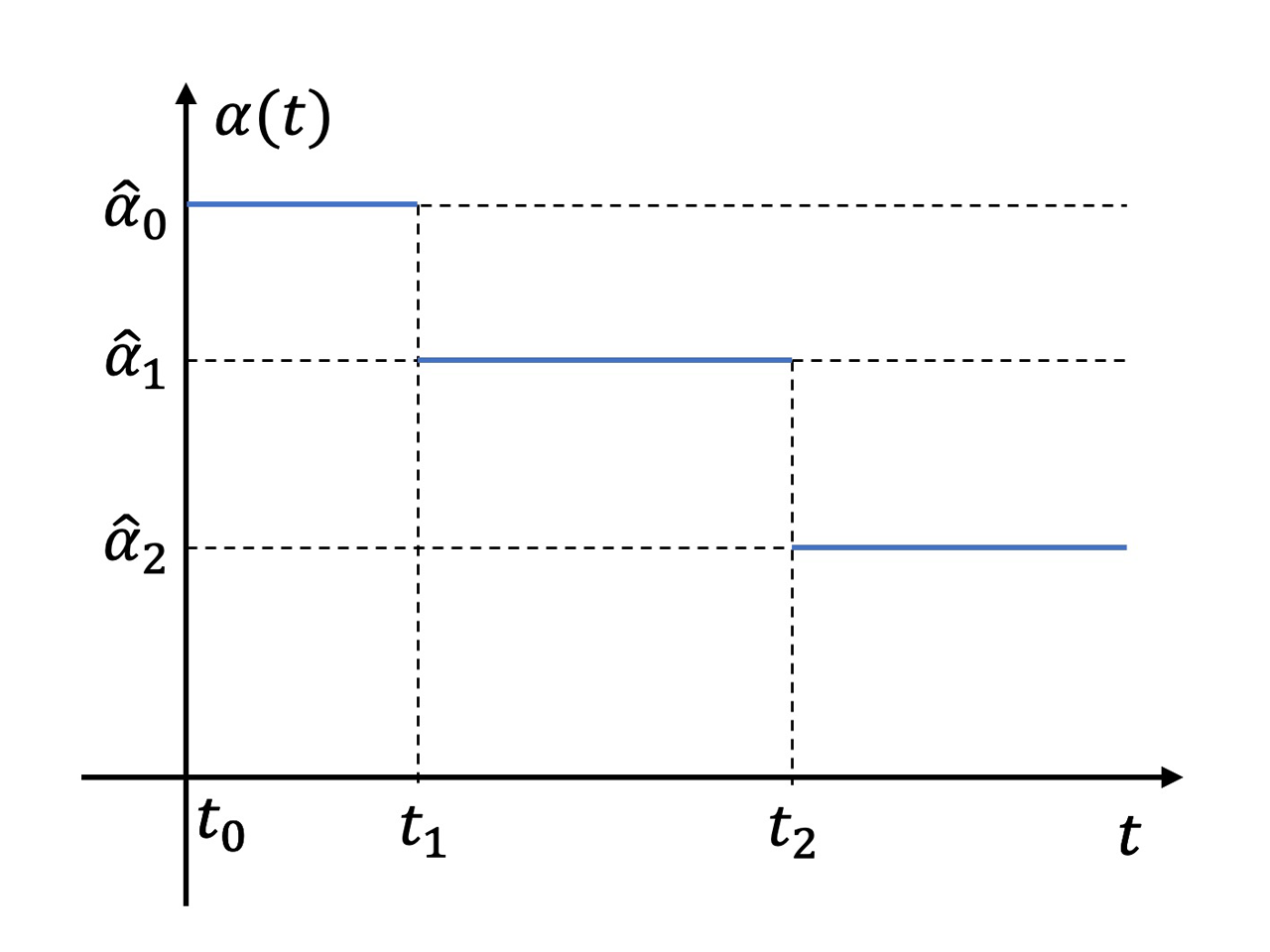}
\end{center}
\vspace{-0.25in}
\caption{The behavior of the angular velocity (\textbf{left}) and the learning rate (\textbf{right}) for Algorithm~\ref{alg:LRdrop}. The derivative threshold $\tau_i\!=\!\min\{\tau_0, \gamma \hat{\alpha}_i/2\}$.}
\label{fig:AutoLR}
\vspace{-0.1in}
\end{wrapfigure}
$v_\alpha(t)$ saturates in $\frac{\pi}{2}[1+\epsilon\alpha]$ when $t$ goes to infinity. Note that the given model complies with the property P2 empirically observed and described in Section~\ref{sec:ME}: i) if the learning rate is large enough, the angular velocity saturates at a level larger than $\pi/2$ and smaller than $2\pi/3$; ii) as the learning rate decreases, the angular velocity saturates at progressively lower levels; iii) smaller learning rate leads to a slower saturation of angular velocity; iv) when the learning rate is low enough the angular velocity saturates at $\pi/2$. Lets assume an upper-bound $\alpha_{max}$ for the learning rate. Since the limit of the angular velocity should be between $\pi/2$ and $2\pi/3$, the range of factor $\epsilon$ is set to be $(0,\frac{1}{3\alpha_{max}})$.

For the the purpose of the theoretical analysis, we drop the learning rate every time the derivative of the angular velocity decreases to a threshold $\tau_i$ (Algorithm \ref{alg:LRdrop}) instead of detecting whether the change of the angular velocity is small enough (Algorithm \ref{alg:ALRD}). Intuitively, when the derivative of the angular velocity is close to zero, we would expect the angular velocity to saturate. We are going to analyze the convergence of Algorithm \ref{alg:LRdrop}, which is an approximate version of Algorithm~\ref{alg:ALRD}. The behavior of the angular velocity and the learning rate for Algorithm~\ref{alg:LRdrop} is depicted in Figure~\ref{fig:AutoLR}.

\begin{thm}\label{thm:conv}
Suppose $f(x)$ is a convex function, $\mathbb{E}\left[\norm{ \mathcal{G}(x;\xi)-\mathbb{E}[\mathcal{G}(x;\xi)]}\right]\leq\delta^2$ and $\norm{\partial f(x)}\leq G$ for any $x$ and some non-negative $G$. Given the sequence of the learning rates $\{\hat{\alpha}_i\}_{i=-1}^{n-1}$ such that $\hat{\alpha}_i=(i+1)^{-\frac{2}{3}}$, parameters $\epsilon\in(0,\frac{1}{3\hat{\alpha}_0})$ and $\gamma$ defining the angular velocity model $v_{\alpha}(t)$ (Equation~\ref{eq:AVModel}), and the initial threshold $\tau_0$ ($\tau_0<2$) for the derivative of the angular velocity, the sequence of weights $\{x_t\}_{t=0}^{T-1}$ generated by Algorithm~\ref{alg:LRdrop} satisfies
\vspace{-0.05in}
\begin{align}
    \min_{t=0,...,T-1}\{\mathbb{E}[f(x_t)-f(x^*)]\}\leq&\frac{\beta(f(x_0)\!-\!f(x^*))[\log\left(\sqrt{\frac{2T}{\kappa_1}}\!+\!1\right)\!-\!\log 2]}{\kappa_1(1-\beta)\left[\sqrt{\frac{2T}{\kappa_2}}-3\right]}+\frac{(1-\beta)\norm{x_0-x^*}^2}{2\kappa_1\left[\sqrt{\frac{2T}{\kappa_2}}-3\right]}\notag\\
    &+\frac{(2s\beta+1)(G^2+\delta^2)\kappa_2\log \left(\sqrt{\frac{2T}{\kappa_1}}\right)}{2(1-\beta)\kappa_1\left[\sqrt{\frac{2T}{\kappa_2}}-3\right]}\\
    =&O\left(\log T/\sqrt{T}\right),
\end{align}
\vspace{-0.2in}

where $\kappa_1=\frac{\sqrt{\pi}-1}{\gamma}$ and $\kappa_2=\frac{1}{\gamma}\sqrt{2\pi/3\tau_0}$.
\end{thm}
Theorem~\ref{thm:conv} can be obtained by extending Theorem~\ref{thm:sgdm_conv} to the setting accommodating the angular velocity model from Equation~\ref{eq:AVModel} and  guarantees sub-linear convergence rate of Algorithm~\ref{alg:LRdrop}.

\section{Experiments}
\label{sec:ER}

\begin{wraptable}{r}{0.6\textwidth} 
\vspace{-0.15in}
\caption{Test errors of AutoDrop and baselines reported in the literature. For CIFAR-$10$ and CIFAR-$100$ we ran each experiment four times with different random seeds. We report the mean and standard deviation of the final test error (at the $200^{\text{th}}$ epoch). For ImageNet, we ran each experiment once and report the final test error (at the $105^{\text{th}}$ epoch). $^{\dagger}$ follows the the setup of \citep{zhang2019lookahead}. $^{\ddagger}$ follows the the setup of \citep{Zagoruyko2016WRN}. $^{*}$ follows the the setup of \citep{He2016DeepRL}.} 
\vspace{-0.1in}
\label{tab: cifar}
\centering
\begin{tabular}{|p{1.8cm}||p{3cm}|p{2cm}|}
\hline
Model&Method &Test Error [\%]\\
\hline
\multirow{4}{8em}{ResNet-$18$ CIFAR-$10$} 
&Baseline$^{\dagger}$ ($\rho=0.2$)  &$4.87\pm0.085$\\
&AutoDrop ($\rho=0.1$) &$5.07\pm0.465$\\
&AutoDrop ($\rho=0.2$) &$\mathbf{4.61 \pm 0.173}$\\
&AutoDrop ($\rho=0.5$) &$\mathbf{4.71\pm0.111}$\\
\hline

\multirow{4}{8em}{WRN-$28$x$10$ CIFAR-$10$}
&Baseline$^{\ddagger}$ ($\rho=0.2$) & $3.77 \pm 0.05$\\
&AutoDrop ($\rho=0.1$) & $\mathbf{3.73 \pm 0.26}$\\
&AutoDrop ($\rho=0.2$) & $\mathbf{3.73 \pm 0.10}$ \\
&AutoDrop ($\rho=0.5$) & $4.29 \pm 1.13$\\
\hline

\multirow{4}{8em}{ResNet-$34$ CIFAR-$100$}
&Baseline$^{\dagger}$ ($\rho=0.2$)  & $21.91 \pm 0.20$ \\
&AutoDrop ($\rho=0.1$) & $23.27 \pm 0.48$ \\
&AutoDrop ($\rho=0.2$) & $\mathbf{21.82 \pm 0.50}$ \\
&AutoDrop ($\rho=0.5$) & $\mathbf{21.43 \pm 0.29}$ \\
\hline

\multirow{4}{8em}{WRN-$40$x$10$ CIFAR-$100$ } 
&Baseline$^{\ddagger}$ ($\rho=0.2$) & $19.16 \pm 0.11$\\
&AutoDrop ($\rho=0.1$) & $\mathbf{18.25 \pm 0.33}$ \\
&AutoDrop ($\rho=0.2$) & $\mathbf{18.17 \pm 0.25}$\\
&AutoDrop ($\rho=0.5$) & $23.15 \pm 3.43$ \\
\hline 

\multirow{2}{8em}{ResNet-18 ImageNet } 
&Baseline$^{*}$ ($\rho=0.1$) &$29.93$ \\
&AutoDrop ($\rho=0.1$) & $\mathbf{29.80}$\\
\hline
\end{tabular}
\vspace{-0.25in}
\end{wraptable}

In this section, we compare the performance of our method, AutoDrop, that automatically adjusts the learning rate, with the SOTA optimization approaches for training DL models that instead manually drop the learning rate. The comparison is performed on the popular DL architectures and benchmark data sets. Our method was run with three different settings of the learning rate drop factor $\rho$, whereas the remaining hyper-parameters were set as recommended in Section~\ref{sec:Alg}. The baselines that we compare with are SOTA approaches taken from the referenced papers that rely on different variants of SGD. Finally, the codes of our method will be publicly released.

In Table~\ref{tab: cifar} we show the final test errors obtained on CIFAR-$10$, CIFAR-$100$, and ImageNet data sets. Our method shows better performance in terms of the final test error compared to the baseline approaches while automatically selecting the epochs for dropping the learning rate. Across all the experiments on CIFAR data sets, AutoDrop run with the learning drop factor $\rho = 0.2$ (the drop factor used by the baselines), was always among the winning AutoDrop strategies. For ImageNet the baseline recommended using $\rho = 0.1$ and again for this setting AutoDrop performed favorably. Furthermore, in Figure~\ref{fig:TC} we report an exemplary plot capturing the behavior of the learning rate, train loss, and test error as a function of the number of epochs. 

\begin{figure}[h]
\vspace{-0.1in}
\centering
\includegraphics[width=.4\textwidth]{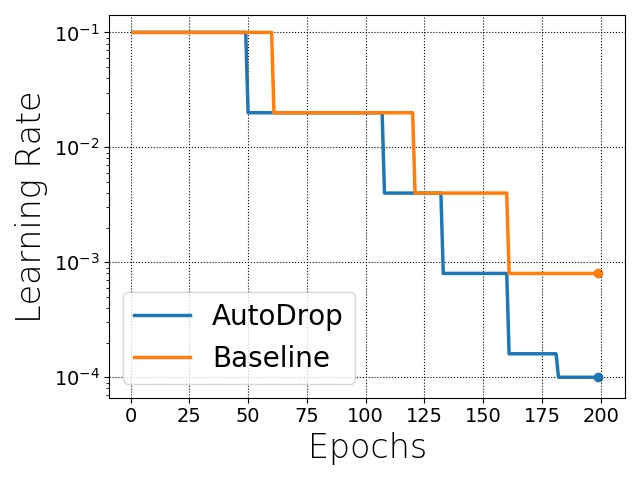} \qquad \qquad
\includegraphics[width=.4\textwidth]{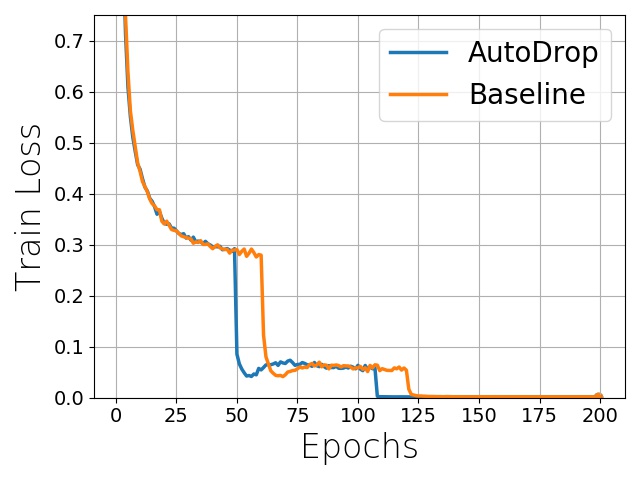} \\
\vspace{-0.05in}
\includegraphics[width=.4\textwidth]{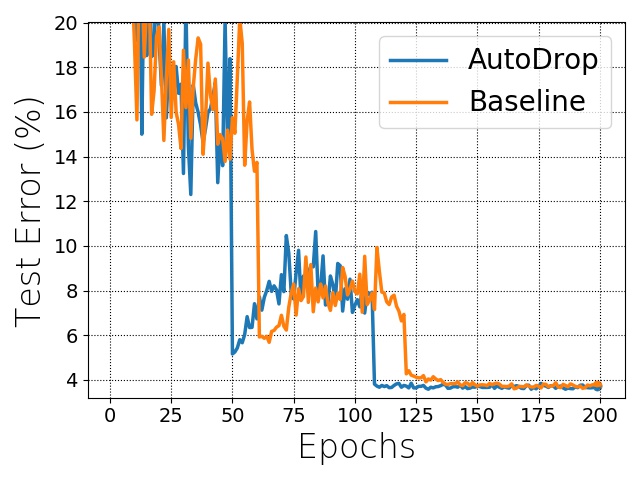} \qquad \qquad
\includegraphics[width=.4\textwidth]{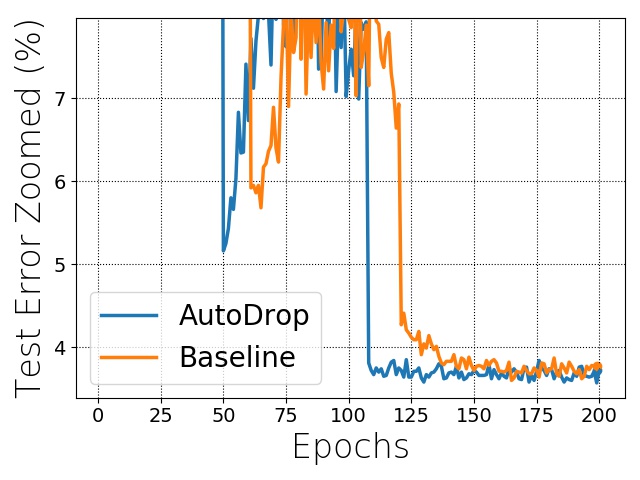}
\vspace{-0.15in}
\caption{Experimental curves for WRN-$28$x$10$ model and CIFAR-$10$ data set: learning rate, train loss, test error, and zoomed test error.}
\label{fig:TC}
\vspace{-0.2in}
\end{figure}

\begin{minipage}{.45\textwidth}
\vspace{-0.15in}
\caption{Test error [\%] of AutoDrop and the baseline for different initial learning rates. Resnet-$18$ on CIFAR-$10$.}
\vspace{-0.15in}
\label{tab:lrs}
\centering
\begin{tabular}{|c||c|c|}
\hline
Initial &Baseline &AutoDrop \\
LR & ($\rho = 0.2$) & ($\rho = 0.2$)\\
\hline
{$0.15$} & $5.04 \pm 0.19$ & $\mathbf{4.90 \pm 0.20}$\\\hline
{$0.1$}  & $4.80 \pm 0.12$ & $\mathbf{4.62 \pm 0.14}$\\\hline
{$0.05$} & $4.87 \pm 0.09$ & $\mathbf{4.61 \pm 0.17}$\\\hline
{$0.03$} & $5.07 \pm 0.28$ & $\mathbf{4.73 \pm 0.25}$\\\hline
\end{tabular}
\vspace{-0.1in}
\end{minipage}
\begin{minipage}{0.54\textwidth}
\vspace{-0.1in}
\centering
\includegraphics[width=0.65\textwidth]{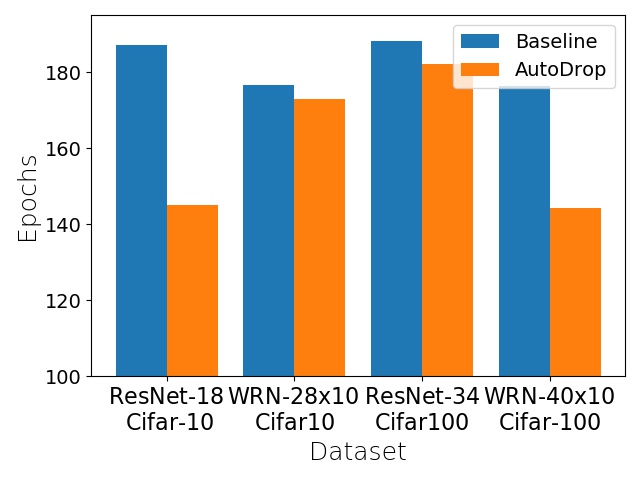}
\vspace{-0.17in}
\captionof{figure}{The average number of training epochs needed by an optimizer to achieve the lowest test error. AutoDrop and Baseline use $\rho = 0.2$.}
\label{fig:CE}
\end{minipage}

Next, in Table~\ref{tab:lrs} we verify if AutoDrop is more robust to the choice of the initial learning rate than the baseline. We ran an experiment on CIFAR-10 and ResNet-18 and confirmed that indeed across different choices of the initial learning rate, AutoDrop consistently outperforms the baseline. 

Finally, Figure~\ref{fig:CE} is confronting the convergence speed of our method and the baseline by reporting the average number of training epochs needed by AutoDrop and the baseline to achieve the lowest test error. We ran each experiment four times with different random seeds and report the mean value. Clearly, AutoDrop is faster. 

\vspace{-0.05in}
\section{Conclusions}
\label{sec:Con}
\vspace{-0.05in}

This paper is motivated by a growing need to develop DL optimization techniques that are more automated in order to increase their scalability and improve the accessibility to DL technology by a wider range of participants. The selection of hyperparameters for training DL models, and especially the learning rate scheduling, is a very hard problem and still remains largely unsolved in the literature. We provide a new algorithm, AutoDrop, for adjusting the learning rate drop during training of DL models that works online and can be run on the top of any DL optimization scheme. It is furthermore a very simple algorithm to implement and use. AutoDrop enjoys favorable empirical performance compared to SOTA training approaches in terms of test error and convergence speed. Finally, our method has a theoretical underpinning that we show, and enjoys sub-linear convergence. 

\section{Acknowledgement}
The authors would like to acknowledge that the NSF Award number $2041872$ sponsored the research work presented in this paper.

\bibliography{iclr2022_conference}
\bibliographystyle{iclr2022_conference}

\appendix
\input{supplementary.tex}

\end{document}

%% file: math_commands.tex
\usepackage{amsmath,amsfonts,bm}

\def\eqref#1{equation~\ref{#1}}

\def\1{\bm{1}}

\DeclareMathAlphabet{\mathsfit}{\encodingdefault}{\sfdefault}{m}{sl}
\SetMathAlphabet{\mathsfit}{bold}{\encodingdefault}{\sfdefault}{bx}{n}

\newcommand{\norm}[1]{\left\lVert#1\right\rVert}

%% file: supplementary.tex
\newpage
\clearpage

\vskip 0.15in
\vskip -\parskip
\begin{center}
{\LARGE AutoDrop: Training Deep Learning Models with Automatic Learning Rate Drop \\ (Supplementary Material) \par} 
\end{center}
\vskip 0.25in
\vskip -\parskip
\hrule height 1pt

\section{Proof for Theorem \ref{thm_1}}

\begin{proof}[Proof for Theorem \ref{thm_1}]
  First note that if the learning rate is chosen as specified, then each of the trajectories is a contraction map. By Banach’s fixed point theorem, they each have a unique fixed point. Clearly 
  $$\mathbb{E}_{SGD}^*=\lim_{t\to\infty}\mathbb{E}[x_t]=0.$$
  For the variance we can solve for the fixed points directly. Define $\mathbb{V}_{SGD}^*=\lim_{t\to\infty}\mathbb{V}[x_t]$,
  \begin{align*}
      &\mathbb{V}_{SGD}^*=(I-\gamma A)^2\mathbb{V}_{SGD}^*+\gamma A^2\Sigma,\\
      \Longrightarrow&\mathbb{V}_{SGD}^*=\frac{\gamma^2 A^2\Sigma}{I-(I-\gamma A)^2}=diag(\frac{\alpha^2a_1^2\sigma_1^2}{1-(1-\alpha a_1)^2},\cdots,\frac{\alpha^2a_n^2\sigma_n^2}{1-(1-\alpha a_n)^2}),
  \end{align*}
  where $\sigma_i^2$ is the i-th diagonal element of the variance matrix $\Sigma$ of a gaussian noise $c_t$.
  Because
  \begin{align*}
      \mathbb{V}_{SGD}^*=\lim_{t\to\infty}\mathbb{V}[x_t]&=\lim_{t\to\infty}\mathbb{E}\left[(x_t-\mathbb{E}[x_t])(x_t-\mathbb{E}[x_t])^T\right]\\\notag
       &=\lim_{t\to\infty}\mathbb{E}[x_tx_t^T]\\\notag
      &=diag(\lim_{t\to\infty}\mathbb{E}[x_{t,1}^2],\lim_{t\to\infty}\mathbb{E}[x_{t,2}^2],\cdots,\lim_{t\to\infty}\mathbb{E}[x_{t,n}^2]),
  \end{align*}
  we have
  \begin{align}\label{eq:x}
      \lim_{t\to\infty}\mathbb{E}[x_{t,i}^2]=\frac{\alpha^2a_i^2\sigma_i^2}{1-(1-\alpha a_i)^2}\quad i=1,\cdots,n.
  \end{align}
  Since $c_t\sim N(0,\Sigma)$, 
  \begin{align}\label{eq:c}
      \lim_{t\to\infty}\mathbb{E}[c_{t,i}^2]=\sigma_i^2\quad i=1,\cdots,n.
  \end{align}
  The update formula with learning rate $\alpha$ is
  \begin{align}
      x_{t+1}=x_t-\alpha\nabla\hat{L}(x_t)= x_t-\alpha A(x_t-c_t),\quad c_t\sim N(0,\Sigma).
  \end{align}
  For the next iteration, the update formula can be written as
  \begin{align}
      x_{t+2}&=x_{t+1}-\alpha\nabla\hat{L}(x_{t+1})\\\notag
      &=x_{t+1}-\alpha A(x_{t+1}-c_{t+1}),\quad c_{t+1}\sim N(0,\Sigma)\\\notag
      &=x_{t+1}-\alpha A(x_t-\alpha A(x_t-c_t)),\quad c_t,c_{t+1}\sim N(0,\Sigma)\\\notag
      &=x_{t+1}-\alpha A(x_t-c_{t+1})+\alpha^2A^2(x_t-c_t),\quad c_t,c_{t+1}\sim N(0,\Sigma).
  \end{align}
  Define the step at iteration t as $s_t=x_{t+1}-x_t$, then the inner product of two consecutive steps can be written as
  \begin{align}
      <s_t,s_{t+1}>=&<-\alpha A(x_t-c_t), -\alpha A(x_t-c_{t+1})+\alpha^2A^2(x_t-c_t)>\\\notag
      =&\alpha^2(x_t-c_t)^TA^2(x_t-c_{t+1})-\alpha^3(x_t-c_t)^TA^3(x_t-c_t)\\\notag
      =&\alpha^2\left[x_t^TA^2x_t\!-\!x_t^TA^2c_{t+1}\!-\!c_t^TA^2x_t\!+\!c_t^TA^2c_{t+1}\!-\!\alpha x_t^TA^3x_t\!+\!2\alpha x_t A^3c_t\!-\!\alpha c_t^TA^3c_t\right].
  \end{align}
  Therefore, the trajectory of the expectation of the inner product converges to
  \begin{align}\label{eq:dot}
      I^*=\lim_{t\to\infty}\mathbb{E}[<s_t,s_{t+1}>]&=\alpha^2\left[\lim_{t\to\infty}\mathbb{E}[x_t^TA^2(I-\alpha A)]x_t-\alpha\lim_{t\to\infty}\mathbb{E}[c_t^TA^3c_t]\right]\\\notag
      &=\alpha^2\left[\sum_{i=1}^na_i^2(1-\alpha a_i)\lim_{t\to\infty}\mathbb{E}[x_{t,i}^2]-\sum_{i=1}^n\alpha a_i^3\lim_{t\to\infty}\mathbb{E}[c_{t,i}^2]\right]\\\notag
      &=\alpha^2\sum_{i=1}^n\left[a_i^2(1-\alpha a_i)\frac{\alpha a_i\sigma_i^2}{2-\alpha a_i}-\alpha a_i^3\sigma_i^2\right]\\\notag
      &=\alpha^2\sum_{i=1}^n\alpha a_i^3\sigma_i^2\left[\frac{1-\alpha a_i}{2-\alpha a_i}-1\right]\\\notag
      &=-\alpha^3\sum_{i=1}^n\frac{a_i^3\sigma_i^2}{2-\alpha a_i}.
  \end{align}
  The norm of step $s_t$ at iteration t is written as
  \begin{align}
      \norm{s_t}^2&=\norm{\alpha A(x_t-c_t)}^2\\\notag
      &=\alpha^2(x_t-c_t)^TA^2(x_t-c_t)\\\notag
      &=\alpha^2(x_t^TA^2x_t-2x_t^TA^2c_t+c_t^TA^2c_t).
  \end{align}
  Therefore the trajectory of the expectation of the norm of $s_t$ converges to
  \begin{align}\label{eq:norm}
      N^*=\lim_{t\to\infty}\mathbb{E}[\norm{s_t}^2]&=\alpha^2\lim_{t\to\infty}\mathbb{E}[x_t^TA^2x_t]+\alpha^2\lim_{t\to\infty}\mathbb{E}[c_t^TA^2c_t]\\\notag
      &=\alpha^2\sum_{i=1}^na_i^2\left(\mathbb{E}[x_{t,i}^2]+\mathbb{E}[c_{t ,i}^2]\right)\\\notag
      &=\alpha^2\sum_{i=1}^na_i^2\sigma^2\left(\frac{\alpha a_i}{2-\alpha a_i}+1\right)\\\notag
      &=2\alpha^2\sum_{i=1}^n\frac{a_i^2\sigma^2}{2-\alpha a_i}.
  \end{align}
  Here, in order to draw meaningful conclusions we make certain simplifications and proceed by approximating $\mathbb{E}[cos(\angle(s_t,s_{t+1}))]\approx\mathbb{E}[<s_t,s_{t+1}>]/\mathbb{E}[\norm{s_t}\norm{s_{t+1}}]$. 
  
  Because $cos(\angle(s_t,s_{t+1}))=\frac{<s_t,s_{t+1}>}{\norm{s_t}\norm{s_{t+1}}}$ and $\norm{s}_t$ converges when t is large enough, then
  \begin{align}\label{eq:cos}
      \lim_{t\to\infty}\mathbb{E}[cos(\angle(s_t,s_{t+1}))]\approx\lim_{t\to\infty}\frac{\mathbb{E}[<s_t,s_{t+1}>]}{\mathbb{E}[\norm{s_t}^2]}.
  \end{align}
  
  Since $I^*=\lim_{t\to\infty}\mathbb{E}[cos(\angle(s_t,s_{t+1}))]$ and $N^*=\lim_{t\to\infty}\mathbb{E}[\norm{s_t}^2]$ are both bounded and not equal to 0,
  \begin{align}\label{eq:cos}
      \lim_{t\to\infty}\mathbb{E}[cos(\angle(s_t,s_{t+1}))]\approx\frac{\lim_{t\to\infty}\mathbb{E}[<s_t,s_{t+1}>]}{\lim_{t\to\infty}\mathbb{E}[\norm{s_t}^2]}.
  \end{align}
  By combining formula (\ref{eq:cos}), (\ref{eq:dot}) and (\ref{eq:norm}), we obtain that the expectation of cosine value converges to
  \begin{align}
      C^*\!\!=\!\!\lim_{t\to\infty}\mathbb{E}[cos(\angle(s_t,s_{t+1}))]\!\approx\!\frac{I^*}{N^*}&\!=\!-\frac{\alpha}{2}\frac{\sum_{i=1}^n\frac{a_i^3\sigma_i^2}{2-\alpha a_i}}{\sum_{i=1}^n\frac{a_i^2\sigma_i^2}{2-\alpha a_i}}\!\geq\!-\frac{\alpha}{2}\max_i a_i\frac{\sum_{i=1}^n\frac{a_i^2\sigma_i^2}{2-\alpha a_i}}{\sum_{i=1}^n\frac{a_i^2\sigma_i^2}{2-\alpha a_i}}\!=\!-\frac{\alpha\max_i a_i}{2}
  \end{align}
  Since $I-\alpha A\succ 0$ implies $\alpha a_i<1$ for arbitrary $i$, then $C^*\in [-\frac{1}{2}, 0]$ and the angle is between 90 degree to 120 degrees.
\end{proof}

\section{Proof for Theorem \ref{thm:sgdm_conv}}
Proof in this section in inspired by \cite{yang2016unified}.

\begin{proof}[Proof for Theorem 2]
 We denote $\mathcal{G}(x_t;\xi_t)=\mathcal{G}(x_t)=\mathcal{G}_t$. The update formula (\ref{eq:sgdm}) implies the following recursions:
\begin{align}
    x_{t+1}+p_{t+1}=&x_t+p_t-\frac{\alpha_t}{1-\beta}\mathcal{G}(x_t)\\
    v_{t+1}=&\beta v_t+((1-\beta)s-1)\alpha_t\mathcal{G}(x_t),
\end{align}
where $v_t=\frac{1-\beta}{\beta}p_t$ and $p_t$ is given by
\begin{equation}
p_t=\left\{
\begin{aligned}\label{eq:sgdm3}
      &\frac{\beta}{1-\beta}(x_t-x_{t-1}+s\alpha_{t-1}\mathcal{G}(x_{t-1})), \quad k\geq1\\
      &0, \quad k=0
\end{aligned}
\right.  .  
\end{equation}
Define $\delta_t=\mathcal{G}_t-\partial f(x_t)$ and let $x^*$ be the optimal point. From the above recursions we have
\begin{align}
    &\norm{x_{t+1}+p_{t+1}-x^*}^2\notag\\
    =&\norm{x_t+p_t-x^*}^2\!-\!\frac{2\alpha_t}{1-\beta}(x_t+p_t-x^*)^T\mathcal{G}_t\!+\!\left(\frac{\alpha_t}{1-\beta}\right)^2\norm{\mathcal{G}_t}^2\notag\\
    =&\norm{x_t+p_t-x^*}^2\!-\!\frac{2\alpha_t}{1-\beta}(x_t-x^*)^T\mathcal{G}_t\!-\!\frac{2\alpha_t\beta}{(1-\beta)^2}(x_{t}-x_{t-1})^T\mathcal{G}_t\notag\\
    &-\frac{2s\alpha_t\alpha_{t-1}\beta}{(1-\beta)^2}\mathcal{G}_{t-1}^T\mathcal{G}_t\!+\!\left(\frac{\alpha_t}{1-\beta}\right)^2\norm{\mathcal{G}_t}^2\notag\\
    =&\norm{x_t+p_t-x^*}^2-\frac{2\alpha_t}{1-\beta}(x_t-x^*)^T(\delta_t+\partial f(x_t))-\frac{2\alpha_t\beta}{(1-\beta)^2}(x_{t}-x_{t-1})^T(\delta_t+\partial f(x_t))\notag\\
    &-\frac{2s\alpha_t\alpha_{t-1}\beta}{(1-\beta)^2}(\delta_{t-1}+\partial f(x_{t-1}))^T(\delta_t+\partial f(x_t))+\left(\frac{\alpha_t}{1-\beta}\right)^2\norm{\delta_t+\partial f(x_t)}^2.
\end{align}
Note that
\begin{align*}
    &\mathbb{E}[(x_t-x^*)^T(\delta_t+\partial f(x_t))]=\mathbb{E}[(x_t-x^*)^T\partial f(x_t)]\\
    &\mathbb{E}[(x_{t}-x_{t-1})^T(\delta_t+\partial f(x_t))]=\mathbb{E}[(x_{t}-x_{t-1})^T\partial f(x_t)]\\
    &\mathbb{E}[(\delta_{t-1}+\partial f(x_{t-1}))^T(\delta_t+\partial f(x_t))]=\mathbb{E}[(\delta_{t-1}+\partial f(x_{t-1}))^T\partial f(x_t)]=\mathbb{E}[\mathcal{G}_{t-1}^T\partial f(x_t)]\\
    &\mathbb{E}[\norm{\delta_t+\partial f(x_t)}^2]=\mathbb{E}[\norm{\delta_t}^2]+\mathbb{E}[\norm{\partial f(x_t)}^2].
\end{align*}
Taking the expectation on both sides gives the following
\begin{align}\label{eq:norm_exp2}
    &\mathbb{E}[\norm{x_{t+1}+p_{t+1}-x^*}^2]\notag\\
    =&\mathbb{E}[\norm{x_t+p_t-x^*}^2]-\frac{2\alpha_t}{1-\beta}\mathbb{E}[(x_t-x^*)^T\partial f(x_t)]-\frac{2\alpha_t\beta}{(1-\beta)^2}\mathbb{E}[(x_{t}-x_{t-1})^T\partial f(x_t)]\notag\\
    &-\frac{2s\alpha_t\alpha_{t-1}\beta}{(1-\beta)^2}\mathbb{E}[\mathcal{G}_{t-1}^T\partial f(x_t)]+\left(\frac{\alpha_t}{1-\beta}\right)^2(\mathbb{E}[\norm{\delta_t}^2]+\mathbb{E}[\norm{\partial f(x_t)}^2]).
\end{align}
Moreover, since f is convex,$\mathbb{E}\left[\norm{ \mathcal{G}(x;\xi)-\mathbb{E}[\mathcal{G}(x;\xi)]}\right]\leq\delta^2$, and $\norm{\nabla f(x)}\leq G$, then for any $x$
\begin{align*}
    &f(x_t)-f(x^*)\leq(x_t-x^*)^T\partial f(x_t)\\
    &f(x_t)-f(x_{t-1})\leq(x_t-x_{t-1})^T\partial f(x_t)\\
    &-\mathbb{E}[\mathcal{G}_{t-1}^T\partial f(x_t)]\leq\frac{\mathbb{E}[\norm{\mathcal{G}_{t-1}}^2+\norm{\partial f(x_t)}^2]}{2}\leq\delta^2/2+G^2\leq\delta^2+G^2\\
    &\mathbb{E}[\norm{\delta_t}^2]\leq\delta^2,\quad\mathbb{E}[\norm{\partial f(x_t)}^2]\leq G^2.
\end{align*}
Therefore, (\ref{eq:norm_exp2}) can be rewritten as
\begin{align}\label{eq:exp_update2}
    \mathbb{E}[\norm{x_{t+1}+p_{t+1}-x^*}^2]\leq&\mathbb{E}[\norm{x_{t}+p_{t}-x^*}^2]-\frac{2\alpha_t}{1-\beta}\mathbb{E}[f(x_t)-f(x^*)]\\\notag
    &-\frac{2\alpha_t\beta}{(1-\beta)^2}\mathbb{E}[f(x_t)-f(x_{t-1})]+\frac{2s\beta\alpha_t\alpha_{t-1}+\alpha_t^2}{(1-\beta)^2}(G^2+\delta^2).
\end{align}
Since $\hat{\alpha}_i$ is decreasing, it implies that $\alpha_t$ is non-increasing. Thus, (\ref{eq:exp_update2}) could be upper-bounded as
\begin{align}\label{eq:exp_update2}
    \mathbb{E}[\norm{x_{t+1}+p_{t+1}-x^*}^2]\leq&\mathbb{E}[\norm{x_{t}+p_{t}-x^*}^2]-\frac{2\alpha_t}{1-\beta}\mathbb{E}[f(x_t)-f(x^*)]\\\notag
    &-\frac{2\alpha_t\beta}{(1-\beta)^2}\mathbb{E}[f(x_t)-f(x_{t-1})]+\frac{(2s\beta+1)\alpha_t\alpha_{t-1}}{(1-\beta)^2}(G^2+\delta^2).
\end{align}
Taking $t=0,...,T-1$ and $x_{-1}=x_0$, and then summing all the inequalities gives
\begin{align*}
    \sum_{t=0}^{T-1}\mathbb{E}[\norm{x_{t+1}\!+\!p_{t+1}\!-\!x^*}^2]\leq&\sum_{t=0}^{T-1}\mathbb{E}[\norm{x_{t}+p_{t}-x^*}^2]-\sum_{t=0}^{T-1}\frac{2\alpha_t}{1-\beta}\mathbb{E}[f(x_t)-f(x^*)]\notag\\
    &-\!\sum_{t=0}^{T-1}\frac{2\alpha_t\beta}{(1\!-\!\beta)^2}\mathbb{E}[f(x_t)\!-\!f(x_{t-1})]
    \!+\!\frac{(2s\beta\!+\!1)(G^2\!+\!\delta^2)}{(1\!-\!\beta)^2}\sum_{t=0}^{T-1}\alpha_t\alpha_{t-1}.
\end{align*}
Therefore,
\begin{align*}
    \frac{2}{1\!-\!\beta}\sum_{t=0}^{T-1}\alpha_t\mathbb{E}[f(x_t)\!-\!f(x^*)]\leq&\norm{x_0\!-\!x^*}^2\!-\!\norm{x^{T}\!+\!p_{T}\!-\!x^*}\!+\!\frac{2\beta}{(1\!-\!\beta)^2}\sum_{t=0}^{T-1}\alpha_t\mathbb{E}[f(x_{t-1})\!-\!f(x_t)]\\
    &+\frac{(2s\beta+1)(G^2+\delta^2)}{(1-\beta)^2}\sum_{t=0}^{T-1}\alpha_t\alpha_{t-1},
\end{align*}
since $\alpha_{T-1}\leq...\leq\alpha_1\leq\alpha_0<1$, $\min_{t=0,...,T-1}\{\mathbb{E}[f(x_t)-f(x^*)]\}\leq\mathbb{E}[f(x_t)-f(x^*)](\forall t=0,...,T-1)$. Then
\begin{align*}
    \frac{2}{1-\beta}\min_{t=0,...,T-1}\{\mathbb{E}[f(x_t)-f(x^*)]\}\sum_{t=0}^{T-1}\alpha_t\leq&\norm{x_0-x^*}^2+\frac{2\beta}{(1-\beta)^2}\sum_{t=0}^{T-1}\alpha_t\mathbb{E}[f(x_{t-1})-f(x_t)]\notag\\
    &+\frac{(2s\beta+1)(G^2+\delta^2)\sum_{t=0}^{T-1}\alpha_t\alpha_{t-1}}{(1-\beta)^2}.
\end{align*}
Moreover, $\alpha_t=\hat{\alpha}_i (t_i\leq t< t_{i+1})$ implies that
\begin{align*}
    \frac{2}{1\!-\!\beta}\min_{t=0,...,T-1}\{\mathbb{E}[f(x_t)\!-\!f(x^*)]\}\sum_{t=0}^{T-1}\alpha_t\leq&\norm{x_0\!-\!x^*}^2\!+\!\frac{2\beta}{(1\!-\!\beta)^2}\sum_{i=0}^{n-1}\hat{\alpha}_i\mathbb{E}[f(x_{t_i})\!-\!f(x_{t_{i+1}})]\notag\\
    &+\frac{(2s\beta+1)(G^2+\delta^2)\sum_{t=0}^{T-1}\alpha_t\alpha_{t-1}}{(1-\beta)^2}.
\end{align*}
Since $\mathbb{E}[f(x_{t_i})-f(x_{t_{i+1}})]$ is always upper-bounded by $f(x_0)-f(x^*)$, we have
\begin{align*}
    \frac{2}{1-\beta}\min_{t=0,...,T-1}\{\mathbb{E}[f(x_t)-f(x^*)]\}\sum_{t=0}^{T-1}\alpha_t\leq&\norm{x_0-x^*}^2+\frac{2\beta}{(1-\beta)^2}[f(x_0)-f(x^*)]\sum_{i=0}^{n-1}\hat{\alpha}_i\\
    &+\frac{(2s\beta+1)(G^2+\delta^2)\sum_{t=0}^{T-1}\alpha_t\alpha_{t-1}}{(1-\beta)^2}.
\end{align*}
After simplification, we have
\begin{align}\label{ieq:min2}
    \min_{t=0,...,T-1}\{\mathbb{E}[f(x_t)-f(x^*)]\}\leq&\frac{(1-\beta)\norm{x_0-x^*}^2}{2\sum_{t=0}^{T-1}\alpha_t}+\frac{\beta[f(x_0)-f(x^*)]\sum_{i=0}^{n-1}\hat{\alpha}_i}{(1-\beta)\sum_{t=0}^{T-1}\alpha_t}\notag\\
    &+\frac{(2s\beta+1)(G^2+\delta^2)\sum_{t=0}^{T-1}\alpha_t\alpha_{t-1}}{2(1-\beta)\sum_{t=0}^{T-1}\alpha_t}.
\end{align}
Because $\hat{\alpha}_i\leq (i+2)^{-1}$, $k_i\hat{\alpha}_i\geq \kappa_1(i+2)^{-\frac{1}{3}},$ $k_i\hat{\alpha}_i\hat{\alpha}_{i-1}\leq \kappa_2(i+1)^{-\frac{2}{3}},\forall i=0,1,...,n-1(n\gg1)$,
\begin{align}
    \sum_{i=0}^{n-1}\hat{\alpha}_i&\leq\sum_{i=0}^{n-1}(i+2)^{-1}=\int_{0}^{n-1}(i+2)^{-1}=\log(n+1)-\log(2)\label{eq:sumi_2}\\
    \sum_{t=0}^{T-1}\alpha_t&=\sum_{i=0}^{n-1}k_i\hat{\alpha}_i\geq\sum_{i=0}^{n-1}\kappa_1=\kappa_1n\label{eq:sumt}\\
    \sum_{t=0}^{T-1}\alpha_t\alpha_{t-1}&\leq\sum_{i=0}^{n-1}k_i\hat{\alpha}_i\hat{\alpha}_{i-1}\leq\kappa_2\sum_{i=0}^{n-1}(i+1)^{-1}=\kappa_2\int_{0}^{n-1}(i+1)^{-1}=\kappa_2\log n\label{eq:sumt2_2}.
\end{align}
Substituting (\ref{eq:sumi_2}-\ref{eq:sumt2_2}) into inequality (\ref{ieq:min2}) gives 
\begin{align*}
    \min_{t=0,...,T-1}\{\mathbb{E}[f(x_t)-f(x^*)]\}\leq&\frac{\beta(f(x_0)-f(x^*))[\log(n+1)-\log 2]}{\kappa_1(1-\beta)n}+\frac{(1-\beta)\norm{x_0-x^*}^2}{2\kappa_1n}\notag\\
    &+\frac{(2s\beta+1)(G^2+\delta^2)\kappa_2\log n}{2(1-\beta)\kappa_1n}.
\end{align*}
\end{proof}

\subsection{Proof for Theorem \ref{thm:conv}}
First, we introduce Lemma~\ref{lma:k_bound2} which will be used in the proof for Theorem \ref{thm:conv}. We prove this lemma later in this section.

\begin{lma}\label{lma:k_bound2}
If sequences $\{\hat{\alpha}_i\}_{i=-1}^{n-1}\subset(0,1)$ and $\{k_i\}_{i=0}^n\subset\mathbb{N}$ satisfy:
\begin{align*}
    \hat{\alpha}_i=(i+2)^{-1},\quad \frac{\kappa_1}{\hat{\alpha}_i}\leq k_i\leq\frac{\kappa_2}{\hat{\alpha}_i},
\end{align*}
where $\kappa_1$, $\kappa_2$ are constants, then 
\begin{align}\label{ieq:ki_con2_2}
    \hat{\alpha}_i\leq (i+2)^{-1}, \quad k_i\hat{\alpha}_i\geq \kappa_1,\quad k_i\hat{\alpha}_i\hat{\alpha}_{i-1}\leq \kappa_2(i+1)^{-1},\quad \forall i=0,1,...,n-1.
\end{align}
Moreover, suppose $T=\sum_{i=0}^{n-1}k_i$. If $n\gg 1$ the following holds
\begin{align}\label{eq:T2}
    \frac{\kappa_1n(n+3)}{2}\leq T\leq\frac{\kappa_2n(n+3)}{2}.
\end{align}
\end{lma}

\begin{proof}[Proof for Theorem \ref{thm:conv}]
The derivative of the angular velocity model is:
$$v_{\alpha}'(t)=\frac{\pi(1+\epsilon\alpha)}{2\gamma\alpha(t+1/\gamma\alpha)^2}.$$
Define the gaps of partition $\Pi:0=t_0<t_1<...<t_n=T$ derived from the Algorithm \ref{alg:LRdrop} as
$$k_i=t_{i+1}-t_i,\quad\forall i=0,...,n-1.$$
Since we drop the learning rate every time the derivative of the angular velocity is smaller that the threshold $\tau_i=\min\{\tau_0,\gamma\hat{\alpha}_i/2\}$, we have
\begin{align*}
    v_{\hat{\alpha}_i}'(k_i)=\tau_i\Longrightarrow k_i=(\gamma\hat{\alpha}_i)^{-\frac{1}{2}}\left[\sqrt{\frac{\pi(1+\epsilon\hat{\alpha}_i)}{2\tau_i}}-(\gamma\hat{\alpha}_i)^{-\frac{1}{2}}\right].
\end{align*}

\begin{itemize}
    \item[i)] From $\tau_i=\min\{\tau_0,\gamma\hat{\alpha}_i/2\}$, we have $\tau_i\leq\gamma\hat{\alpha}_i/2$. Therefore,
    \begin{align}
        k_i\geq&(\gamma\hat{\alpha}_i)^{-\frac{1}{2}}\left[\sqrt{\pi(1+\epsilon\hat{\alpha}_i)}\times(\gamma\hat{\alpha}_i)^{-\frac{1}{2}}-(\gamma\hat{\alpha}_i)^{-\frac{1}{2}}\right]\notag\\
        =&\frac{1}{\gamma\hat{\alpha}_i}\left[\sqrt{\pi(1+\epsilon\hat{\alpha}_i)}-1\right]\notag\\
        \geq&\frac{\sqrt{\pi}-1}{\gamma \hat{\alpha}_i}
    \end{align}
    \item[ii)] From $\tau_i=\min\{\tau_0,\gamma\hat{\alpha}_i/2\}$, we have
    \begin{align}
        k_i\leq&(\gamma\hat{\alpha}_i)^{-\frac{1}{2}}\times \sqrt{\frac{\pi(1+\epsilon\hat{\alpha}_i)}{2\tau_i}}\notag\\
        =&(\gamma\hat{\alpha}_i)^{-\frac{1}{2}}\max\left\{\sqrt{\frac{\pi(1+\epsilon\hat{\alpha}_i)}{2\tau_0}},\sqrt{\pi(1+\epsilon\hat{\alpha}_i)}(\gamma\hat{\alpha}_i)^{-\frac{1}{2}}\right\}\notag\\
        =&\max\left\{(\gamma\hat{\alpha}_i)^{-\frac{1}{2}}\sqrt{\frac{\pi(1+\epsilon\hat{\alpha}_i)}{2\tau_0}},\frac{1}{\gamma\hat{\alpha}_i}\sqrt{\pi(1+\epsilon\hat{\alpha}_i)}\right\}\notag\\
        \leq&\frac{1}{\gamma\hat{\alpha}_i}\max\left\{\sqrt{\frac{\pi(1+\epsilon\hat{\alpha}_i)}{2\tau_0}},\sqrt{\pi(1+\epsilon\hat{\alpha}_i)}\right\}\notag.
    \end{align}
    Since $\epsilon\in(0,\frac{1}{3\hat{\alpha}_0})$ and $\tau_0<2$, we could conclude
    \begin{align*}
        k_i\leq\frac{1}{\gamma\hat{\alpha}_i}\max\left\{\sqrt{\frac{2\pi}{3\tau_0}},\sqrt{\frac{4\pi}{3}}\right\}\leq\frac{1}{\gamma\hat{\alpha}_i}\sqrt{\frac{2\pi}{3\tau_0}}.
    \end{align*}
\end{itemize}
Combine i) and ii), we have
\begin{align}
    \frac{\sqrt{\pi}-1}{\gamma}\times\frac{1}{\hat{\alpha}_i}\leq k_i\leq\frac{1}{\gamma}\sqrt{\frac{2\pi}{3\tau_0}}\times\frac{1}{\hat{\alpha}_i}.
\end{align}
Define $\kappa_1=\frac{\sqrt{\pi}-1}{\gamma}$ and $\kappa_2=\frac{1}{\gamma}\sqrt{\frac{2\pi}{3\tau_0}}$. By Lemma \ref{lma:k_bound2}, we have
\begin{align}\label{ieq:ki_co3_2}
    \hat{\alpha}_i\leq (i+2)^{-1}, \quad k_i\hat{\alpha}_i\geq \kappa_1,\quad k_i\hat{\alpha}_i\hat{\alpha}_{i-1}\leq \kappa_2(i+1)^{-1},\quad \forall i=0,1,...,n-1.
\end{align}
Then, by combining (\ref{ieq:ki_co3_2}) with Theorem \ref{thm:sgdm_conv} we could conclude
that the sequence $\{x_t\}_{t=0}^{T-1}$ generated by the Algorithm \ref{alg:LRdrop} satisfies
\begin{align}\label{ieq:conv_n2}
    \min_{t=0,...,T-1}\{\mathbb{E}[f(x_t)-f(x^*)]\}\leq&\frac{\beta(f(x_0)-f(x^*))[\log(n+1)-\log 2]}{\kappa_1(1-\beta)n}+\frac{(1-\beta)\norm{x_0-x^*}^2}{2\kappa_1n}\notag\\
    &+\frac{(2s\beta+1)(G^2+\delta^2)\kappa_2\log n}{2(1-\beta)\kappa_1n}.
\end{align}
By Equation (\ref{eq:T2}) in Lemma \ref{lma:k_bound2} we have that
\begin{align*}
    \frac{\kappa_1n(n+3)}{2}\leq T\leq\frac{\kappa_2n(n+3)}{2}.
\end{align*}
Therefore
\begin{align}\label{ieq:n_bound2}
   \sqrt{\frac{2T}{\kappa_2}}-3\leq n\leq \sqrt{\frac{2T}{\kappa_1}}.
\end{align}
Combining (\ref{ieq:n_bound2}) with (\ref{ieq:conv_n2}) gives
\begin{align*}
    \min_{t=0,...,T-1}\{\mathbb{E}[f(x_t)-f(x^*)]\}\leq&\frac{\beta(f(x_0)-f(x^*))[\log\left(\sqrt{\frac{2T}{\kappa_1}}+1\right)-\log 2]}{\kappa_1(1-\beta)\left[\sqrt{\frac{2T}{\kappa_2}}-3\right]}+\frac{(1-\beta)\norm{x_0-x^*}^2}{2\kappa_1\left[\sqrt{\frac{2T}{\kappa_2}}-3\right]}\notag\\
    &+\frac{(2s\beta+1)(G^2+\delta^2)\kappa_2\log \left(\sqrt{\frac{2T}{\kappa_1}}\right)}{2(1-\beta)\kappa_1\left[\sqrt{\frac{2T}{\kappa_2}}-3\right]}\\
    =&O\left(\frac{\log T}{\sqrt{T}}\right)
\end{align*}
\end{proof}

\subsection{Proof for Lemma \ref{lma:k_bound2}}
\begin{proof}[Proof for Lemma \ref{lma:k_bound2}]
 First, we show bounds from (\ref{ieq:ki_con2_2}) one by one:
\begin{enumerate}
    \item[i)]  $\hat{\alpha}_i=(i+2)^{-1}\leq(i+2)^{-1}$.
    \item[ii)] $k_i\hat{\alpha}_i\geq\kappa_1$.
    \item[iii)] $k_i\hat{\alpha}_i\hat{\alpha}_{i=1}\leq\kappa_2\hat{\alpha}_{i-1}=\kappa_2(i+1)^{-1}\leq\kappa_2(i+1)^{-1}$.
    
\end{enumerate}
Secondly, we compute $T=\sum_{i=0}^{n-1}k_i$ according to the definition of $k_i$. Because $n\gg1$, the sum of the sequence could be treated as an integral:
    \begin{align*}
        T=\sum_{i=0}^{n-1}k_i\leq\kappa_2\sum_{i=0}^{n-1}\frac{1}{\hat{\alpha}_i}=\kappa_2\sum_{i=0}^{n-1}(i+2)=\frac{\kappa_2n(n+3)}{2},
    \end{align*}
and
    \begin{align*}
        T=\sum_{i=0}^{n-1}k_i\geq\kappa_1\sum_{i=0}^{n-1}\frac{1}{\hat{\alpha}_i}=\kappa_1\sum_{i=0}^{n-1}(i+2)=\frac{\kappa_1n(n+3)}{2}.
    \end{align*}
\end{proof}

\newpage
\section{Experimental Details}
\label{sec:ED}

\subsection{Data sets and models}
\textbf{The CIFAR-$\mathbf{10}$ and CIFAR-$\mathbf{100}$ data sets} \citep{cifar} consist of $50$ K training images, with $10$ and $100$ different classes respectively. For CIFAR-$10$ experiments we used a ResNet-$18$ \citep{He2016DeepRL} and a WRN-$28$x$10$ \citep{Zagoruyko2016WRN} models. For CIFAR-$100$ experiments we used a ResNet-$34$ \citep{He2016DeepRL} and a WRN-$40$x$10$ \citep{Zagoruyko2016WRN} models. We do not use the dropout \citep{srivastava2014dropout} layers for WRN models in our experiments. The implementation involving WRN architecture and CIFAR data set relies on publicly available codes\footnote{https://github.com/meliketoy/wide-resnet.pytorch}.


\textbf{The ImageNet (ILSVRC-$\mathbf{2012}$) data set} \citep{imagenet_cvpr09} consists of $1.2$ M images divided into $1$ K categories. We train a ResNet-$18$ \citep{He2016DeepRL} model. We use model implementation from PyTorch official model zoo\footnote{https://pytorch.org/vision/stable/models.html}.

\subsection{Training setup}
For CIFAR-10 and CIFAR-100 experiments we refer to \citep{zhang2019lookahead} and \citep{Zagoruyko2016WRN} for ResNet and WRN models respectively. For ImageNet experiments we follow the training procedure proposed by \citep{He2016DeepRL}. 

In all our experiments, for the baseline we use the same setting of hyperparameters (including the learning rate schedule) as recommended in the referenced literature.

\newpage
\subsection{Additional results}
\begin{figure}[h]
\vspace{-0.1in}
\centering
\includegraphics[width=.4\textwidth]{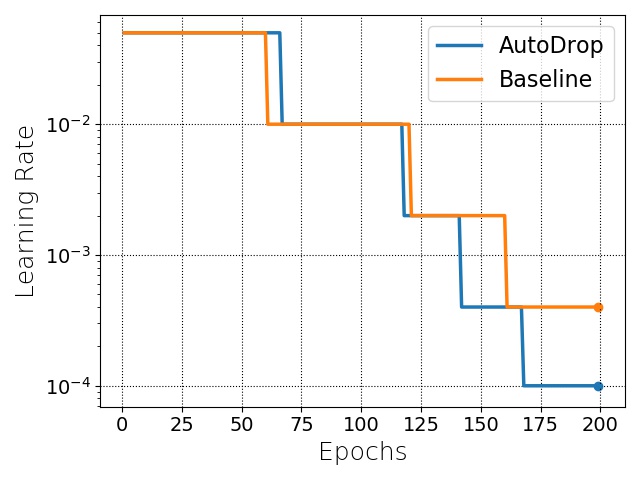}
\includegraphics[width=.4\textwidth]{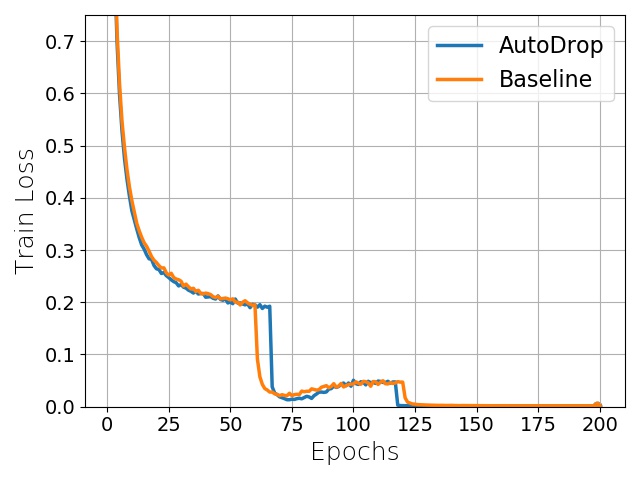}
\includegraphics[width=.4\textwidth]{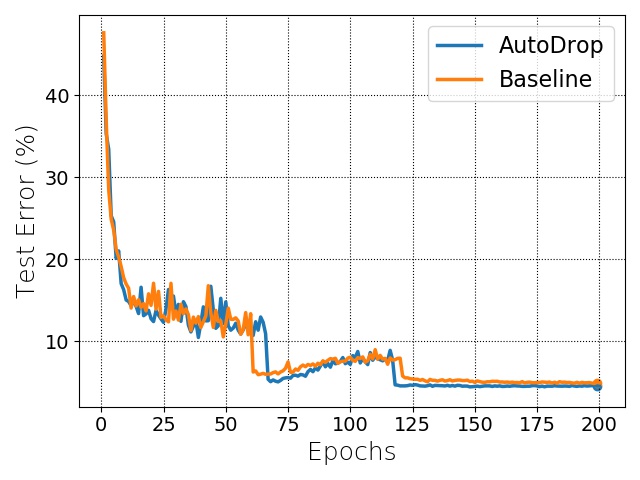}
\includegraphics[width=.4\textwidth]{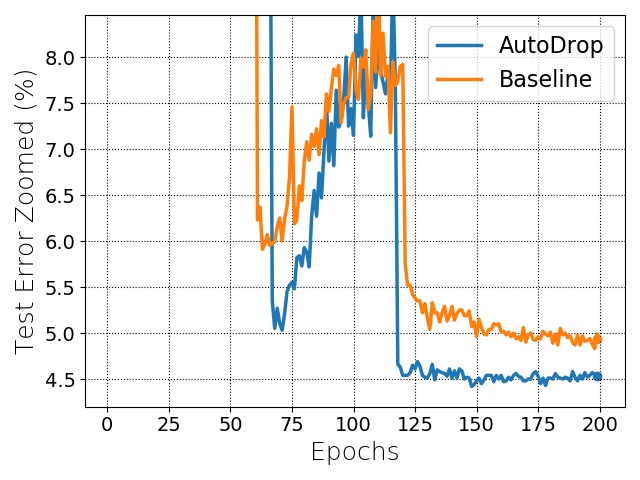}
\vspace{-0.15in}
\caption{Experimental curves for ResNet-$18$ model and CIFAR-$10$ data set. Top (from left to right): learning rate and train loss. Bottom (from left to right): test error and zoomed test error.}
\end{figure}

\begin{figure}[h]
\vspace{-0.1in}
\centering
\includegraphics[width=.4\textwidth]{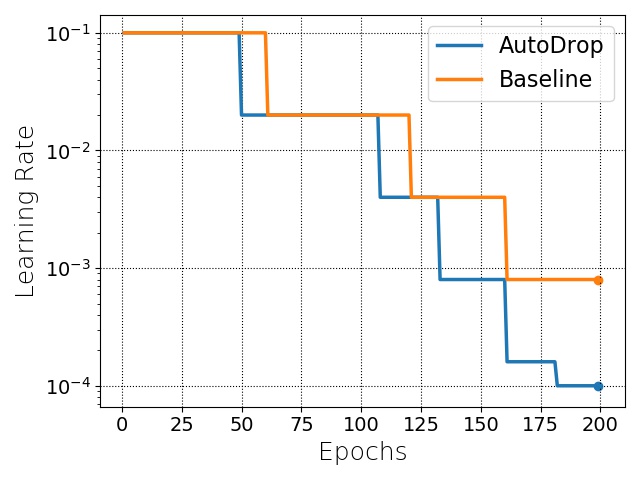}
\includegraphics[width=.4\textwidth]{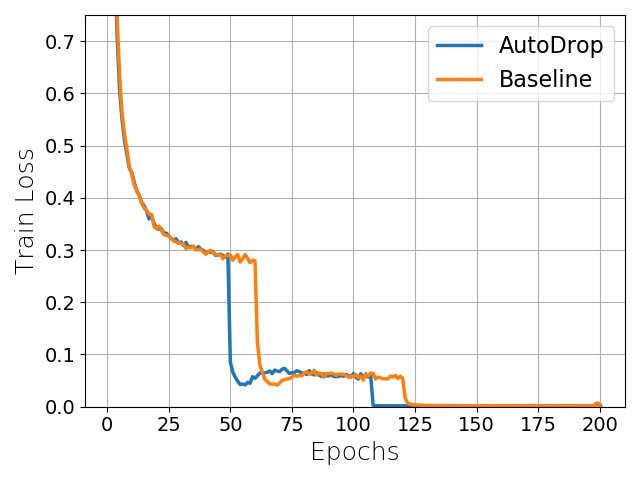}
\includegraphics[width=.4\textwidth]{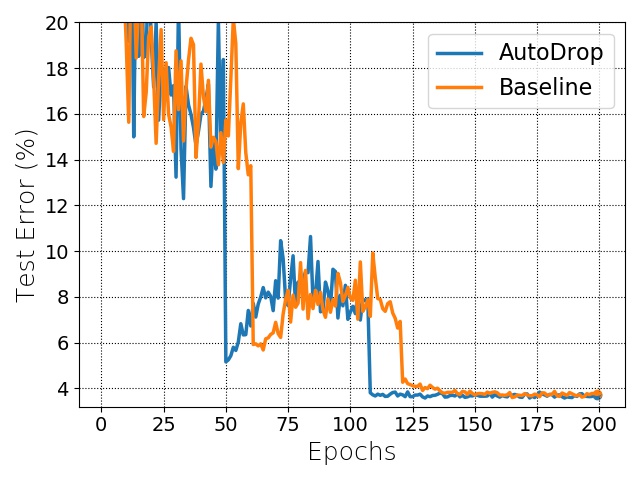}
\includegraphics[width=.4\textwidth]{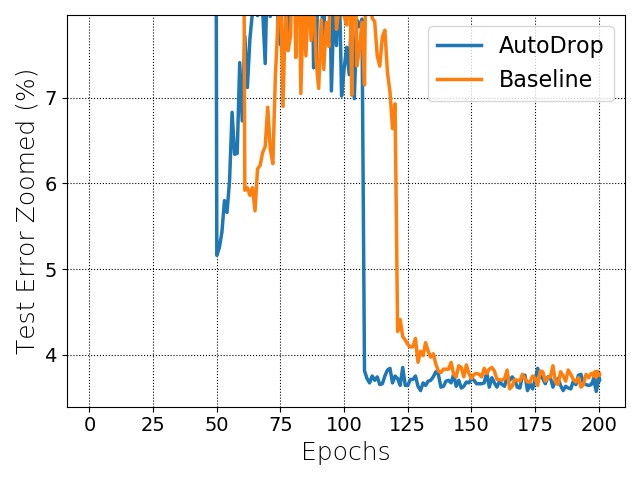}
\vspace{-0.15in}
\caption{Experimental curves for WRN-$28$x$10$ model and CIFAR-$10$ data set. Top (from left to right): learning rate and train loss. Bottom (from left to right): test error and zoomed test error.}
\end{figure}

\begin{figure}
\vspace{-0.1in}
\centering
\includegraphics[width=.4\textwidth]{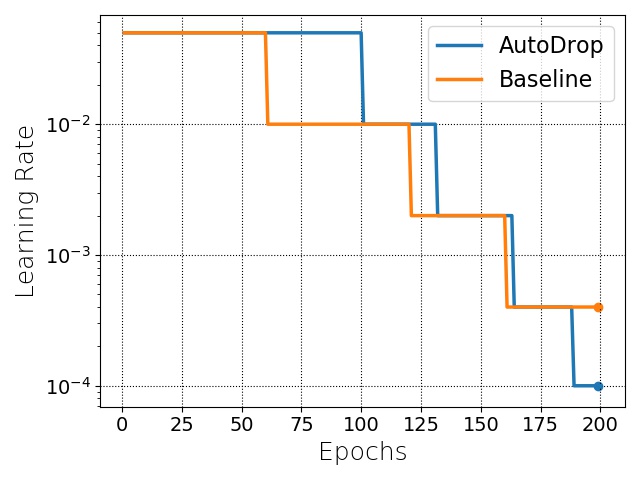}
\includegraphics[width=.4\textwidth]{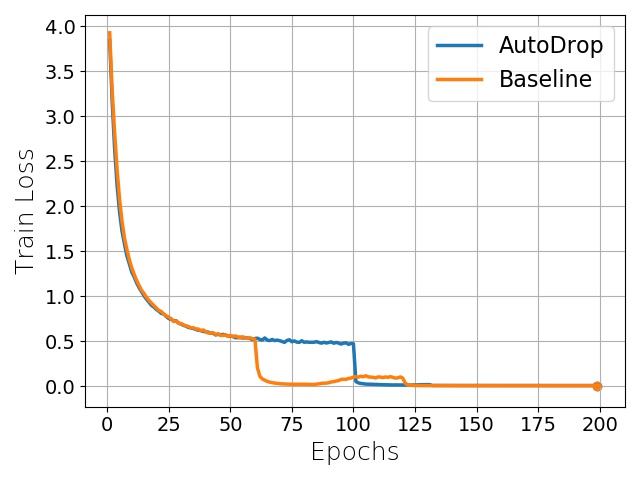}
\includegraphics[width=.4\textwidth]{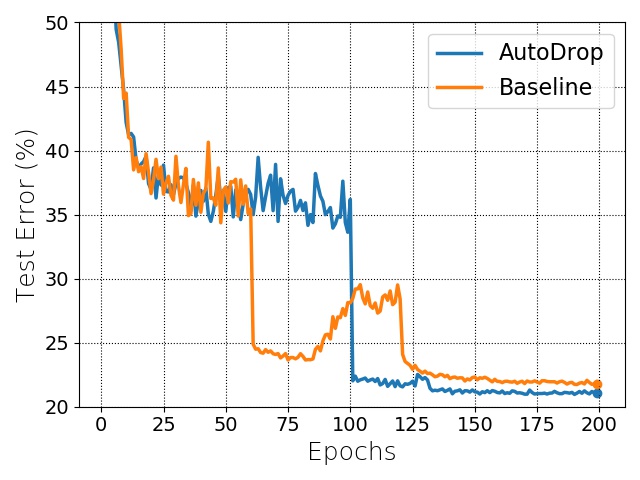}
\includegraphics[width=.4\textwidth]{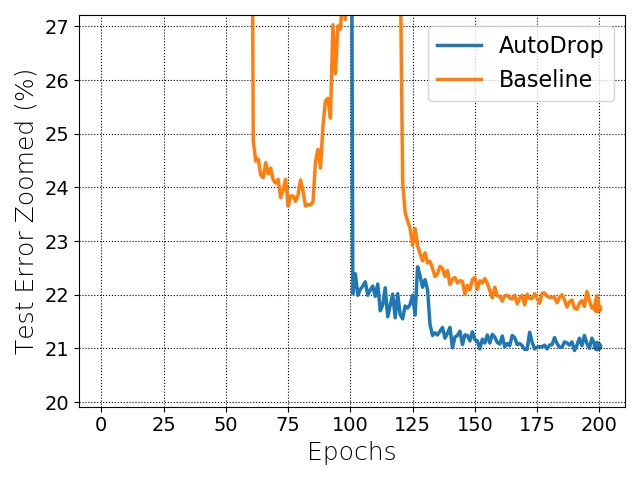}
\vspace{-0.15in}
\caption{Experimental curves for ResNet-$34$ model and CIFAR-$100$ data set. Top (from left to right): learning rate and train loss. Bottom (from left to right): test error and zoomed test error.}
\end{figure}

\begin{figure}
\vspace{-0.1in}
\centering
\includegraphics[width=.4\textwidth]{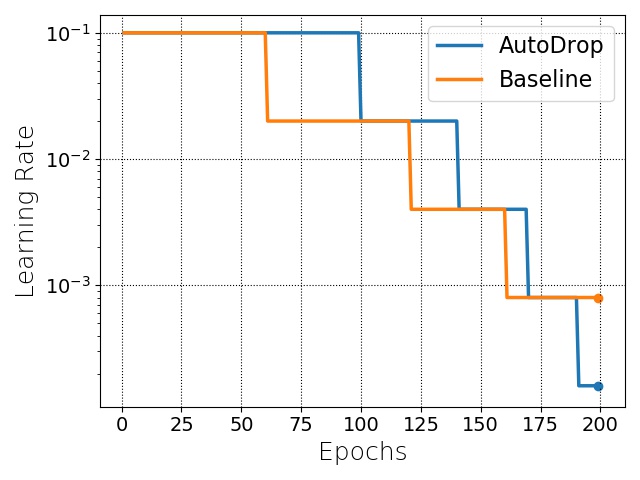}
\includegraphics[width=.4\textwidth]{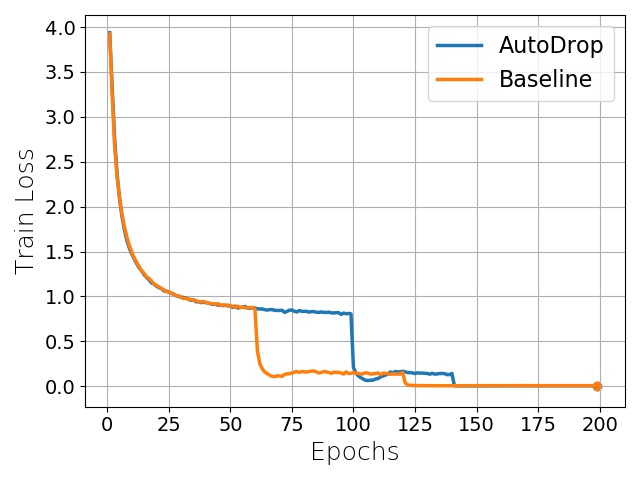}
\includegraphics[width=.4\textwidth]{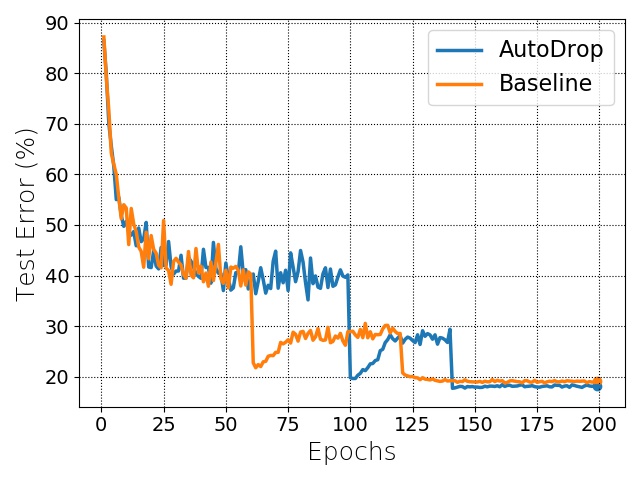}
\includegraphics[width=.4\textwidth]{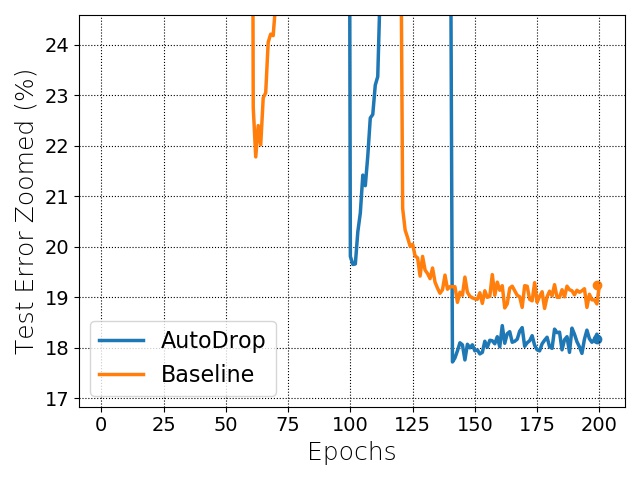}
\vspace{-0.15in}
\caption{Experimental curves for WRN-$40$x$10$ model and CIFAR-$100$ data set. Top (from left to right): learning rate and train loss. Bottom (from left to right): test error and zoomed test error.}
\end{figure}

\newpage
\clearpage

\begin{figure}[htp!]
\vspace{-0.1in}
\centering
\includegraphics[width=.4\textwidth]{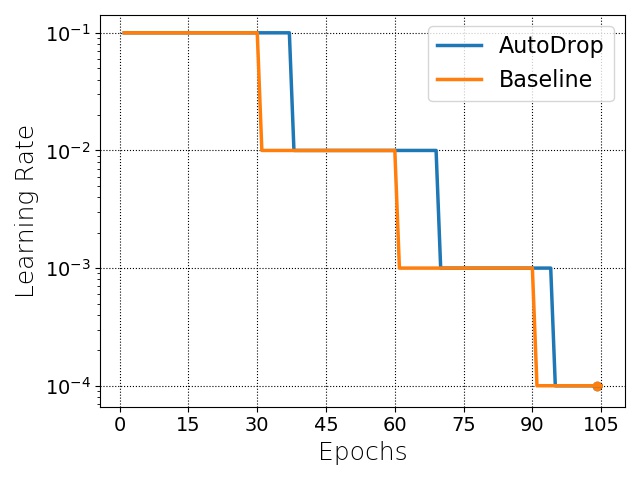}
\includegraphics[width=.4\textwidth]{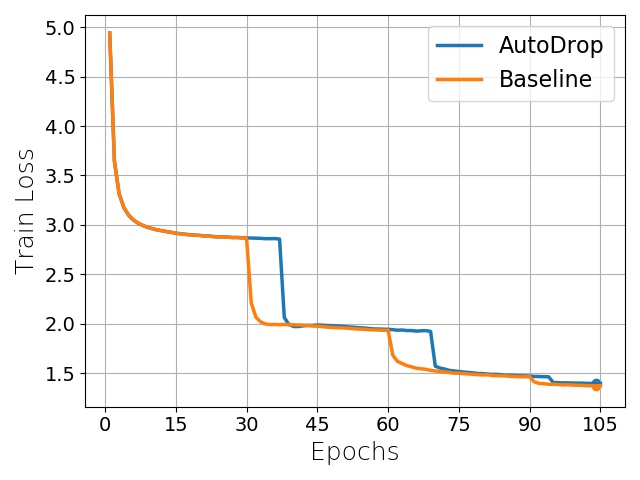}
\includegraphics[width=.4\textwidth]{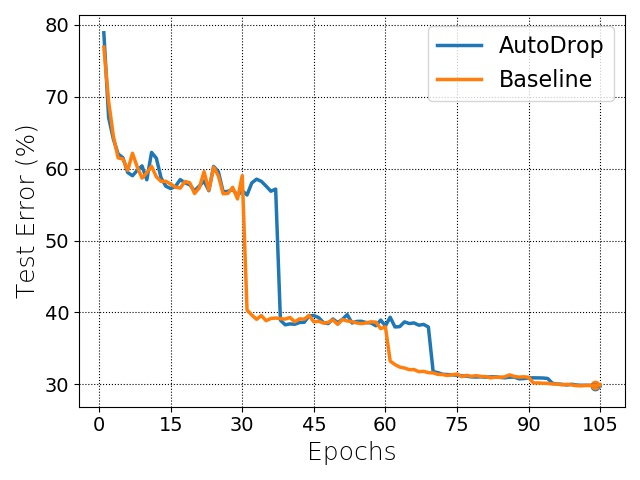}
\includegraphics[width=.4\textwidth]{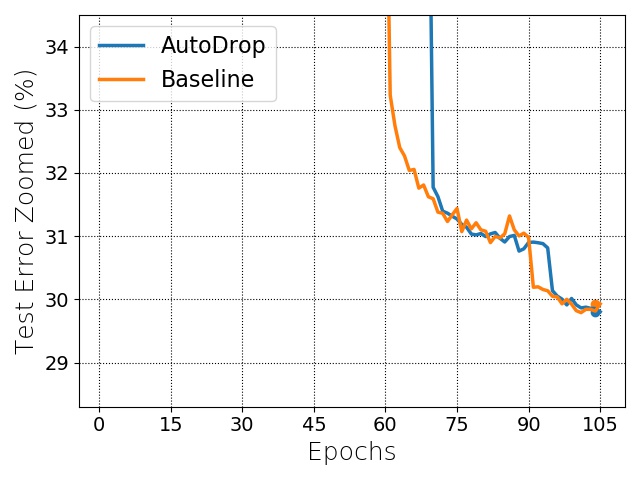}
\vspace{-0.15in}
\caption{Experimental curves for ResNet-$18$ model and ImageNet data set. Top (from left to right): learning rate and train loss. Bottom (from left to right): test error and zoomed test error.}
\end{figure}